\documentclass[conference,compsoc]{IEEEtran}

\usepackage{pgfplots}
\usepackage{comment}
\usepackage[linewidth=0.5pt]{mdframed}
\usepackage{graphicx}
\usepackage{framed}
\usepackage{xspace}
\usepackage{xcolor}
\usepackage{tabularx}
\usepackage{multicol}
\usepackage{multirow}
\usepackage{float}
\usepackage{subcaption}
\usepackage{pifont}
\usepackage{pdfpages}
\usepackage{cite}
\usepackage{hyperref}
\usepackage{bm}

\hypersetup{colorlinks=true,citecolor=blue,filecolor=blue,linkcolor=blue,urlcolor=blue}

\usepackage{colortbl}
\usepackage{enumitem}
\usepackage[override]{cmtt}
\usepackage{graphicx,color}
\usepackage{boxedminipage}
\usepackage{url}
\usepackage{caption}
\usepackage{amsmath,amsthm,amstext,amssymb,amsfonts,latexsym}
\usepackage{algorithm}
\usepackage[noend]{algpseudocode}
\usepackage{algorithmicx}

\usepackage{lipsum}

\newtheorem{definition}{Definition}
\newtheorem{protocol}{Protocol}
\newtheorem{theorem}{Theorem}
\newtheorem{lemma}{Lemma}

\newcommand{\name}{\textsc{FairZK}\xspace}

\renewcommand{\paragraph}[1]{\smallskip\noindent\textbf{#1}}
\newcommand{\ms}[1]{\mathsf{#1}}

\newcommand{\floor}[1]{{\lfloor #1 \rfloor}}
\newcommand{\norm}[1]{\|#1\|}

\newcommand{\mat}[1]{\mathbf{#1}}

\newcommand{\binary}{\{0, 1\}}

\newcommand{\tV}{\tilde{V}}

\newcommand{\tadd}{\tilde{add}}
\newcommand{\tmult}{\tilde{mult}}

\newcommand{\F}{\mathbb{F}}

\newcommand{\V}{\mathcal{V}}
\renewcommand{\P}{\mathcal{P}}
\renewcommand{\S}{\mathcal{S}}

\newcommand{\pp}{\mathsf{pp}}

\newcommand{\pcs}{\mathsf{PCS}}
\newcommand{\Setup}{\mathsf{Setup}}

\newcommand{\Commit}{\mathsf{Commit}}

\newcommand{\Open}{\mathsf{Open}}

\newcommand{\Verify}{\mathsf{Verify}}
\newcommand{\com}{\mathsf{com}}

\newcommand{\inner}[1]{\left< #1 \right>}

\newcounter{mytable}
\def\mytable{\begin{centering}\refstepcounter{mytable}}
\def\endmytable{\end{centering}}

\newcounter{myfig}
\def\myfig{\begin{centering}\refstepcounter{myfig}}
\def\endmyfig{\end{centering}}

\newif\ifcameraready
\camerareadyfalse

\newif\ifeprint
\eprinttrue

\newif\ifshowcomment
 \showcommenttrue
\showcommentfalse 

\ifshowcomment

\newcommand{\yupeng}[1]{{\color{red}[Yupeng: #1]}}

\newcommand{\yanning}[1]{{\color{blue}[Yanning: #1]}}

\else

\newcommand{\yupeng}[1]{}
\newcommand{\yanning}[1]{}

\fi

\newtheorem*{rep@theorem}{\rep@title}
\newcommand{\newreptheorem}[2]{%
\newenvironment{rep#1}[1]{%
 \def\rep@title{#2 \ref{##1}}%
 \begin{rep@theorem}}%
 {\end{rep@theorem}}}
\newreptheorem{theorem}{Theorem}
\newcommand{\new}[1]{#1}

\usepackage{tabularx}
\usepackage{multirow}
\usepackage{booktabs}
\usepackage{amsmath}

\author{
    \IEEEauthorblockN{
        Tianyu Zhang\textsuperscript{*}\IEEEauthorrefmark{2}\IEEEauthorrefmark{4},
        Shen Dong\textsuperscript{*}\IEEEauthorrefmark{2}\IEEEauthorrefmark{4},
        O. Deniz Kose\textsuperscript{*}\IEEEauthorrefmark{3},
        Yanning Shen\IEEEauthorrefmark{3}, 
        Yupeng Zhang\IEEEauthorrefmark{2}
    }
    \\
    \IEEEauthorblockA{
       \IEEEauthorrefmark{2}\textit{University of Illinois, Urbana-Champaign}, \\
        \IEEEauthorrefmark{3}\textit{University of California, Irvine},
       \IEEEauthorrefmark{4}\textit{Shanghai Jiao Tong University}
    }
}

\begin{document}

\title{\name: A Scalable System to Prove Machine Learning Fairness \\in Zero-Knowledge}

\maketitle

\makeatletter
\renewcommand\@makefntext[1]{\noindent\makebox[0em][r]{\@makefnmark\ }#1}
\makeatother

\begingroup
\renewcommand\thefootnote{}
\footnotetext{\textsuperscript{*}Equal contribution. The work was partially done while the first two authors were undergraduate research assistants at UIUC. }
\endgroup

\begin{abstract}

With the rise of machine learning techniques, ensuring the fairness of decisions made by machine learning algorithms has become of great importance in critical applications. However, measuring fairness often requires full access to the model parameters, which compromises the confidentiality of the models. In this paper, we propose a solution using zero-knowledge proofs, which allows the model owner to convince the public that a machine learning model is fair while preserving the secrecy of the model. To circumvent the efficiency barrier of naively proving machine learning inferences in zero-knowledge, our key innovation is a new approach to measure fairness only with model parameters and some aggregated information of the input, but not on any specific dataset. To achieve this goal, we derive new bounds for the fairness of logistic regression and deep neural network models that are tighter and better reflecting the fairness compared to prior work. Moreover, we develop efficient zero-knowledge proof protocols for common computations involved in measuring fairness, including the spectral norm of matrices, maximum, absolute value, and fixed-point arithmetic. 

We have fully implemented our system, \name, that proves machine learning fairness in zero-knowledge. Experimental results show that \name is significantly faster than the naive approach and an existing scheme that use zero-knowledge inferences as a subroutine. The prover time is improved by 3.1$\times$--1789$\times$ depending on the size of the model and the dataset. \name can scale to a large model with 47 million parameters for the first time, and generates a proof for its fairness in 343 seconds. This is estimated to be 4 orders of magnitude faster than existing schemes, which only scale to small models with hundreds to thousands of parameters.

\end{abstract}

\section{Introduction}\label{sec:intro}

Machine learning (ML) algorithms are now deeply integrated into various aspects of our daily lives, including financial services, healthcare systems, and even criminal justice-related prediction tasks. While these models have achieved remarkable advancements, research has demonstrated that ML models propagate possible bias in training data and lead to \emph{unfair} outcomes for downstream applications \cite{awareness,fairdata}. For example, a recent study in~\cite{fuster2022predictably} demonstrated that ML algorithms used in credit risk prediction tasks discriminate against Black and Hispanic borrowers disproportionately when predicting the likelihood of the repayment, which plays a significant role in the approval of loans. In~\cite{angwin2019machine}, it has been shown that the ML models utilized for recidivism prediction across the US (Correctional Offender Management Profiling for Alternative Sanctions (COMPAS)) for who is tend to recommit a crime, or for who is expected to fail to appear at their court hearing are unreliable and racially biased. 

These examples have motivated active research on ML fairness, an important area aiming to mitigate the bias of ML algorithms. Numerous research efforts have focused on assessing the fairness of ML models, e.g., ~\cite{yan2020silva, lalor2024should}, see~\cite{mehrabi2021survey} for a survey. However, almost all existing fairness testing methods require white-box access to the ML model and the training dataset. In practice, the ML models are usually intellectual properties of companies and the dataset often contains sensitive information. It is very challenging to make them available to the public or an auditor to verify the fairness of the ML algorithms. Therefore, there is a great demand for a solution that allows the public verification of the fairness of ML models in the aforementioned applications, while preserving the confidentiality of the models. 

In this paper, we propose to address this problem using the cryptographic primitive of zero-knowledge proofs (ZKPs). A ZKP allows a \emph{prover} to convince a \emph{verifier} that a statement on the prover's secret witness is true, without revealing any additional information about the secret. In the context of ML fairness, the prover's secret is the ML model, and the prover can convince the public that the model is fair without compromising the confidentiality of the model. Unfortunately, applying ZKPs naively would not result in an efficient and scalable solution. Taking the statistical parity~\cite{dwork2012fairness} as an example, which is a commonly utilized \emph{group fairness} metric, given a sensitive attribute $s$ such as gender and race, and two sensitive subgroups with  $s=0$ and $s=1$, it is defined as
\[
\Delta_{SP}=|P(\hat{y}=1 \mid s=0)-P(\hat{y}=1 \mid s=1)|
\]
where $\hat{y}$ denotes the predicted outcome of the binary classification ML model. To show that an ML model is fair under this definition, the prover could use a ZKP protocol to prove that the predictions of the ML model on a public dataset over two sensitive subgroups do not differ by much. However, this would require multiple invocations of ZKPs for the inferences of the model, which is not practical yet for large models and is an active research topic on its own. Moreover, even if the proof can be generated, the fairness only holds on this specific dataset, and may not generalize to other inputs in real-world applications. 

\begin{figure}[t!]
    \centering
    \includegraphics[width=0.49\textwidth]{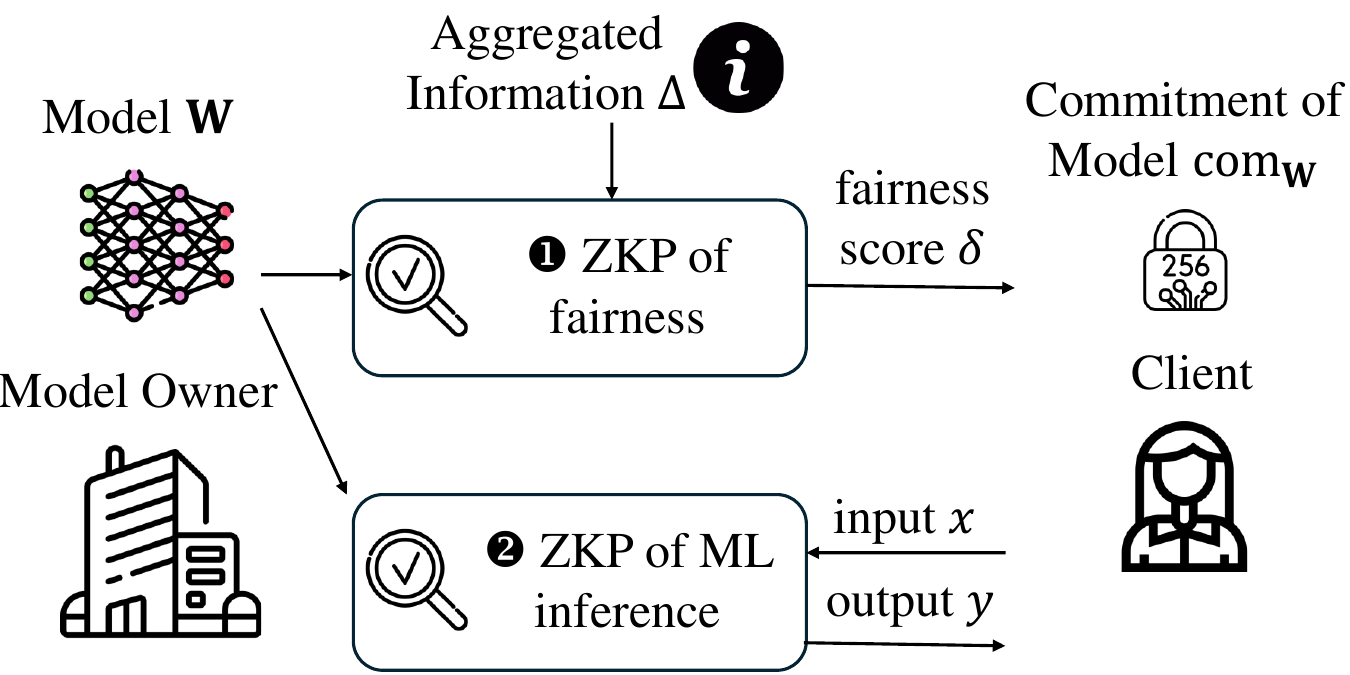}
    \caption{Design of our ZKP of ML fairness.}
    \label{fig:design}
\end{figure}

\paragraph{Our Solution.} In this paper, we propose a fundamentally different approach. The key idea of our solution is to decouple the ZKP of fairness from the naive inferences of the ML model. Ideally, the ZKP only takes the ML model as the secret input of the prover, and evaluates whether the model is fair or not without accessing any specific dataset or data sample. Of course fairness cannot hold universally for any input unless the model is trivial (always output 1 or a random outcome). Therefore, the ZKP also takes some ``aggregated information'' about the data samples that is true for all inputs in the applications. See Section~\ref{subsec:delta} for the formal definitions. For example, the information can be obtained from the aggregated statistics released by the Census Bureau for a region or the entire country. With the ML model and the aggregated information, the fairness of the model is proven in the ZKP.

In this way, the fairness claim holds for any input with such aggregated information, instead of a specific dataset. Moreover, our solution can break the barrier of ZKPs for ML inferences and scale to much larger ML models than those in prior work. Figure~\ref{fig:design} illustrates the design of our solution. The model owner proves to the public that the secret ML model achieves good fairness measured by a ``fairness score'' which we define later, in ZKP using our approach. The verifier can verify the claim with the aggregated information and a commitment of the ML model. This is realized through the commit-and-prove ZKP. Later in real-world applications, for each input, the client can further verify that the decision is indeed made by the same fair model under the commitment. We focus on the first step of proving the fairness of the ML model in ZKP in this paper. In the case where the aggregated information is not publicly available, we further consider a second setting where we can compute it from a large private dataset with a ZKP. See Section~\ref{subsec:setting} for the detailed description of the settings and the threat model.

\paragraph{Our Contributions. } Our contributions can be summarized as follows: 

\begin{itemize}
    \item \textbf{New fairness bounds.} First, to realize our solution, we study the fairness bounds of deep neural networks (DNNs). Crucially, our new bounds only depend on the DNN model weights $\mat{W}$ and aggregated information of the input, but not on any specific data sample or dataset. The techniques are inspired by the bounds for Graph Attention Networks in~\cite{kose2024fairgat}, but we remove the dependency of the bounds on the intermediate values of inferences and obtain tighter bounds than those in~\cite{kose2024fairgat}. 
    See Section~\ref{subsec:diff} for differences from and improvements over~\cite{kose2024fairgat}. As we show in the experiments in Section~\ref{subsec:boundexp}, our new bounds are better than~\cite{kose2024fairgat} and are closely correlated with the fairness of the ML models.
    
    \item \textbf{Specialized ZKP protocols.} Second, we develop efficient ZKP protocols for common operations involved in computing the fairness bounds. Specifically, we propose a new ZKP approach to prove the computation of the spectral norm with a prover time of $O(F^2)$ for a matrix of size $F\times F$. This is significantly faster than the naive approach of computing the spectral norm in ZKP, which incurs a prover time of $O(F^3)$ and is very slow in practice. Our micro-benchmark shows that our new protocol for spectral norm is 5.7--1045$\times$ faster than the naive implementation in a generic ZKP library. 
    For other computations including the absolute value, maximum and fixed-point arithmetic with truncation, we introduce optimizations tailored for our fairness bounds using the recent techniques of lookup arguments.

    \item \textbf{Implementation and evaluations.} Finally, we fully implement our system, \name, for proving the fairness of ML models in zero-knowledge. We benchmark the performance of \name and compare it with prior work and the naive approach of testing the fairness via ZKP of inferences. We demonstrate that \name is $3.1\times$--$1789\times$ faster than the baselines. Moreover, \name can scale to a large DNN with 47 million parameters for the first time. The prover time is only 343 seconds, and is estimated to be 4 orders of magnitude faster than prior work. The code of our implementation is available at \yupeng{update} \url{https://github.com/tnyuzg/FairZK}.
\end{itemize}

\subsection{Related Work}\label{subsec:related}

ZKPs for ML inferences have been studied in the earlier work of~\cite{zhang2020zero,liu2021zkcnn,lee2024vcnn,feng2021zen,weng2021mystique,kang2022scaling,weng2023pvcnn,chen2024zkml,sun2024zkllm}. The most recent scheme named zkLLM in~\cite{sun2024zkllm} scales to proving the inference of a large language model with 13 billion parameters and can generate the proof of one inference in 15 minutes using GPUs. Despite the great progress, the prover time of inference is still high and it would be impractical to naively apply these schemes to prove the fairness of an ML model. Recently, these techniques have been extended to prove the outcome of ML training in~\cite{garg2023experimenting,sun2023zkdl,abbaszadeh2024zero} and with differential privacy in~\cite{shamsabadi2024confidential}. 

The most relevant work to our paper is in~\cite{shamsabadi2022confidential,yadav2024fairproof,franzese2024oath}. Shamsabadi et al.~\cite{shamsabadi2022confidential} proposed a ZKP scheme to prove the fairness of the training of decision trees. Yadav et al.~\cite{yadav2024fairproof} proposed a ZKP scheme to prove the individual fairness of an ML model through the connection to its robustness. Franzese et al.~\cite{franzese2024oath} proposed a probabilistic method to prove the ML fairness through ZKP for inferences. It reduces the number of proofs from one million to 7600 with good robustness of the fairness testing. However, none of the prior work can scale to large ML models used in real-world. For example, the scheme in~\cite{yadav2024fairproof} scales to a small neural network with 2 hidden layers and 130 parameters and generates the proof in several minutes; the prover time of the scheme in~\cite{franzese2024oath} is estimated to be 343,000 seconds to prove the fairness of LeNet with 62,000 parameters. The major overhead comes from the execution of the ZKP for inferences as a subroutine. 

\paragraph{Testing ML fairness.} 
Several tools have been developed in recent years to test the fairness of ML models, such as \cite{chen2024fairness,zhang2020white,zheng2022neuronfair,adebayo2016fairml,tramer2017fairtest,saleiro2018aequitas,bellamy2019ai}. 
These toolkits require full possession of the ML models and full access to datasets to launch the testing. In this paper, we consider proving the fairness without revealing the model using ZKPs.

\section{Preliminaries}\label{sec:prelim}

We use boldface lowercase letters $\bm{a},\bm{b}$ to denote vectors, and boldface uppercase letters $\mat{A},\mat{B}$ to denote matrices. Vectors are column vectors by default. $[n]$ denotes the set $\{0,1,\ldots, n-1\}$. PPT stands for probabilistic polynomial time. $\mathsf{negl}(\lambda)$ denotes the negligible function in the security parameter $\lambda$.

\subsection{Machine Learning and Fairness}\label{subsec:mlprelim}

Multiple notions of fairness have been proposed, including the \emph{group fairness} measures of statistical parity \cite{dwork2012fairness}, equal opportunity \cite{hardt2016equality}, equalized odds \cite{hardt2016equality}, treatment equality \cite{berk2021fairness}, and test fairness \cite{chouldechova2017fair}, and the \emph{individual fairness} measures of unawareness \cite{kusner2017counterfactual, grgic2016case}, awareness \cite{dwork2012fairness}, and counterfactual fairness \cite{kusner2017counterfactual}. We focus on the group fairness in this paper and use the definitions of statistical parity and equal opportunity. Formally speaking, given a binary sensitive attribute $s$ and two sensitive subgroups $\mathcal{S}_0, \mathcal{S}_1$, they are defined as 
\begin{align}
\Delta_{SP}:=&|P(\hat{y}=1 \mid s=0)-P(\hat{y}=1 \mid s=1)| \label{criterion:sp}  \\ 
\Delta_{EO}:=&|P(\hat{y}=1 \mid y=1, s=0)- \nonumber\\
&P(\hat{y}=1 \mid y=1, s=1)|\label{criterion:eo}
\end{align}
where $y$ denotes the ground-truth label and $\hat{y}$ denotes the predicted label. In particular, the statistical parity definition aims to equalize the positive rate in two sensitive groups, thus achieving an equal distribution in the long run, while the equal opportunity definition considers the ground-truth label and tries to achieve similar accuracy across different sensitive groups. The naive approach of removing the sensitive attribute $s$ from the training data does not work, as the biases can still be indirectly introduced via attributes correlated with the sensitive attributes. 

\paragraph{Logistic regression and DNN.} For input data $\bm{x}$ with $F$ features, a logistic regression (LR) model is a vector $\bm{w}$ of dimension $F$ such that the predicted output $\hat{y} = \sigma(\inner{\bm{w},\bm{x}})$, where $\sigma$ is the sigmoid activation function $\sigma(z) = \frac{1}{1+e^{-z}}$, and $z=\inner{\bm{w},\bm{x}}$ denotes the inner product of the two vectors.  

A DNN model of $m$ layers is defined by $m$ matrices of model parameters. We focus on fully-connected neural networks, or multi-layer perceptron (MLP), but we believe our results can be generalized to other DNN models. The dimension of the $\ell$-th matrix $\mat{W}^\ell$ is $F_{\ell+1}\times F_\ell$ for $\ell \in[m]$. The number of features of the input data $\bm{x}$ is $F_0$. Starting from the input layer with $\bm{h}^0 = \bm{x}$, the output of the $\ell$-th fully-connected layer is computed as $\bm{z}^{\ell+1} = \mat{W}^\ell\cdot\bm{h}^\ell$, where $\cdot$ denotes matrix-vector multiplication here. Each fully-connected layer is followed by a nonlinear layer $\bm{h}^{\ell+1} = \sigma(\bm{z}^{\ell+1})$. Here we abuse the notation and use 
$\sigma$ to denote the activation function, which is applied element-wise on $\bm{z}^{\ell+1}$. Common activation functions include the sigmoid function and the ReLU function $\sigma(z) =\max(0,z)$. The output of the last layer is the predicted label of the DNN model. To quantify fairness, we consider binary classification models where the dimension of the last layer is $F_m = 1$ and $\hat{y} = h^m$ is the output of a sigmoid function.  

\paragraph{Lipschitz continuity.} A function $f$ is Lipschitz continuous if there exists a positive real constant $L$ such that, for all $x_1, x_2$, 
\begin{equation}\label{eq:lipschitz}
    |f(x_1)-f(x_2)|\le L |x_1-x_2|.
\end{equation}
$L$ is called the Lipschitz constant of the function. Common activation functions in DNNs are all Lipschitz continuous. For example, $L=0.25$ for the sigmoid function, and $L=1$ for ReLU. 

\paragraph{L2 norm and spectral norm.} The L2 norm (Euclidean norm) of a vector $\bm{x}=(x_0,\ldots, x_{F-1})$ is defined as $\norm{\bm{x}}_2 = \sqrt{\sum_{i=0}^{F-1}x_i^2}$. The norm of a matrix $\mat{W}$ induced by the L2 norm of vectors is called the spectral norm, denoted as $\norm{\mat{W}}_2$ in this paper. By the Cauchy–Schwarz inequality, we have $\norm{\mat{W}\cdot\bm{x}}_2\le \norm{\mat{W}}_2\norm{\bm{x}}_2$. The spectral norm of a matrix $\mat{W}$ equals to the square root of the largest eigenvalue of the square matrix $\mat{W} \cdot \mat{W}^T$, as well as that of $\mat{W}^T\cdot \mat{W}$. 

\subsection{Zero-Knowledge Proofs}

\begin{definition}\label{def:snark}[zero-knowledge Succinct non-interactive argument of knowledge (zkSNARK)] Let $\mathcal{R}$ be a relationship with public instance $x$ and private witnesses $w$, a zkSNARK for $\mathcal{R}$ consists of PPT algorithms ($\mathcal{G}, \mathcal{P}, \mathcal{V}$) with the following properties:

\begin{itemize}
    \item Completeness. For any instance $(x, w)\in \mathcal{R}$

\begin{equation*}    
\Pr \left[ 
\begin{tabular}{ c|c }
 \multirow{2}{6.5em}{$\mathcal{V}(x, vk, \pi)=1$} & $(pk, vk) \gets \mathcal{G}(1^\lambda)$, \\
  & $\pi \gets \mathcal{P}(x, w, pk)$, \\
\end{tabular}
\right]=1 
\end{equation*}

    \item Knowledge Soundness: for any PPT adversary $\mathcal{A}$, there exists an expected PPT knowledge extractor $\mathcal{E}_\mathcal{A}$ such that the following probability is $\le \mathsf{negl}(\lambda)$:

\begin{equation*}    
\Pr \left[ 
\begin{tabular}{ c|c }
 $\mathcal{V}(x, vk, \pi^*)=1$
  & $(pk, vk) \gets \mathcal{G}(1^\lambda)$, \\
  $ \land (x, w)\notin \mathcal{R}$ & $(\pi^*;w) \gets (\mathcal{A}||\mathcal{E}_\mathcal{A})(x, pk)$ \\
\end{tabular}
\right].
\end{equation*}

\item Zero knowledge. There exists a PPT simulator $\mathcal{S}$ such that for any PPT algorithm $\mathcal{V}^*$, $(x,w)\in R$, $pk$ output by $\mathcal{G}(1^\lambda)$, it holds that
\begin{align*}
    & \Pr[\pi\gets\P(x,w,pk): \V^*(x,vk,\pi)=1]\\
    = &\Pr[\pi\gets\mathcal{S}^{\V^*}(x,pk): \V^*(x,vk,\pi)=1], 
\end{align*}
where $\mathcal{S}^{\mathcal{V}^*}(x)$ denotes that $\mathcal{S}$ is given oracle accesses to $\mathcal{V}^*$'s random tape.
 
    \item Succinctness. The proof size $|\pi|$ is sublinear in the size of the relationship $\mathcal{R}$ and the witness $|w|$.

\end{itemize}

\end{definition}

\paragraph{Sumcheck.} The sumcheck protocol is a classical interactive proof protocol proposed by Lund et al. in~\cite{lund1992algebraic}. It allows the prover $\P$ to convince the verifier $\V$ that the summation of a multivariate polynomial $f:\F^k\rightarrow\F$ over the Boolean hypercube is: $$S = \sum\nolimits_{b_1,b_2,\ldots,b_k\in\{0,1\}}f(b_1,b_2,\ldots,b_k).$$
Computing the sum directly by $\V$ will involve $2^k$ evaluations of $f$. The sumcheck protocol allows $\P$ to generate a proof of size $O(k)$ in $k$ rounds where the verifier can verify it efficiently. The formal protocol is presented in Protocol~\ref{prot::sumcheck} in Appendix~\ref{app:sumcheck}. With the algorithms proposed in~\cite{t13,xielibra}, the prover time of a sumcheck protocol on a polynomial $f$ that can be represented as the product of a constant number of multilinear polynomials is linear in the size of the multilinear polynomials. 

\paragraph{Multilinear extension (MLE) and identity polynomial.} 

\begin{definition}[\textbf{Multi-linear Extension}]\label{def:mle}
    Let $V: \{0,1\}^k \rightarrow \mathbb{F}$ be a function. We denote $\tilde{V}: \mathbb{F}^k\rightarrow \F$ as the multilinear extension of V such that the multilinear polynomial $\tilde{V}$ and $V$ agree on the Boolean hypercube: $\tilde{V}(x_1,x_2,\ldots,x_k) = V(x_1,x_2,\ldots,x_k)$ for all $(x_1,x_2,\ldots,x_k)\in \{0,1\}^k$. 
\end{definition}

The closed-form of $\tilde{V}$ can be computed with the MLE of the identity function $eq(x,y)$ such that $eq(x,y) = 1$ if and only if $x = y$ for all $x,y\in\{0,1\}^k$. In particular, $\tilde{eq}(x,y) = \prod_{i=1}^k (x_iy_i+(1-x_i)(1-y_i))$. With this, $\tilde{V}$ is
    \begin{equation*}
        \begin{aligned}
	\tilde{V}(x) = \sum\nolimits_{y\in\{0,1\}^k}\tilde{eq}(x,y) V(y)
	\end{aligned}
    \end{equation*}

\new{\paragraph{GKR and polynomial commitments.} The GKR protocol proposed by Goldwasser et al. in~\cite{GKR} is an interactive proof for layered arithmetic circuits. The prover time is linear in the size of the circuit using the algorithms in~\cite{xielibra}. A zkSNARK can be constructed by combining the GKR protocol with a polynomial commitment scheme (PCS), as proposed in the frameworks in~\cite{hyrax,xielibra,virgo,virgoplus}. We present the description of the zkSNARK based on GKR and PCS in Appendix~\ref{app:sumcheck}.


\paragraph{Lookup arguments.} A common optimization in recent zkSNARKs to support non-arithmetic computations efficiently is the lookup argument. A lookup argument verifies that each element of a query vector is contained within a specified lookup table. Specifically, let $T=(t_0,t_1,\ldots,t_{N-1})\in\mathbb{F}^N$ be the lookup table, and let $A=(a_0,a_1,\ldots,a_{n-1})\in\mathbb{F}^n$ be the result of the lookup. The lookup argument is a zkSNARK for the relation $\mathcal{R}:=\{A\in\mathbb{F}^n:\forall i\in[n], \exists j\in[N] \text{ s.t. } a_i = t_j\}$, ensuring every element in $A$ appears somewhere in $T$. We use the lookup argument named LogUp proposed in~\cite{habock2022multivariate} to be compatible with the sumcheck and the GKR protocols, and we present the construction in Appendix~\ref{app:sumcheck}. We denoted the protocol as $A = \ms{Lookup}(T)$. }


\paragraph{Commitment.} A commitment scheme is a cryptographic primitive with two algorithms:
\begin{itemize}
    \item $\ms{com}\gets\ms{commit}(m,r)$: committing to a message $m$ with the randomness $r$;
    \item $\{0,1\}\gets\ms{Open}(\ms{com},m,r)$: outputting 1 iff $\ms{com} = \ms{commit}(m,r)$. 
\end{itemize}

A commitment scheme satisfies binding and hiding. The binding property ensures that the prover cannot open the commitment to a different message from the committed one, and the hiding property ensures that the commitment leaks no information about the message before the opening. We only use the algorithms in specifying the relationship of our zkSNARK, and we omit the formal definitions. 
\section{New Bounds for ML Fairness}\label{sec:fairness}

In this section, we propose new bounds for fairness of logistic regression and DNN models that are suitable to prove in ZKP under the setting illustrated in Figure~\ref{fig:design}. 

\subsection{Aggregated Information of Dataset}\label{subsec:delta}

We use the following two aggregated information about the input data in our fairness bounds. 

\begin{definition}[Bounded values]\label{def:Delta} For a dataset $\bm{x}$ of dimension $N\times F_0$ with two subgroups with $s=0$ and $s=1$, 
$|x_{j,i} - \bar{x}_i^{(s)}| \le \Delta_{x,i}^{(s)}$, $\forall v_j\in \S_s, s\in\{0,1\}, i\in[F_0]$. Moreover, $\Delta_{x,i} = \max(\Delta_{x,i}^{(0)}, \Delta_{x,i}^{(1)})$. We use $\Delta_x$ to denote vector $(\Delta_{x,i})_{i\in[F_0]}$.  
    
\end{definition}

The definition requires that the value of the $i$-th feature is at most $\Delta_{x,i}$ away from its mean in each subgroup. Similarly, we define the bounded values of each intermediate result of the $\ell$-th layer of DNN as $|z_{j,i}^\ell - \bar{z}_i^{(s),\ell}| \le \Delta_{z,i}^{(s),\ell}$, $\forall v_j\in \S_s, s\in\{0,1\}, i\in [F_\ell]$, and $\Delta_{z,i}^\ell = \max(\Delta_{z,i}^{(0),\ell}, \Delta_{z,i}^{(1),\ell})$. $\Delta_z^\ell$ denotes vector $(\Delta^\ell_{z,i})_{i\in[F_\ell]}$.

\begin{definition}[Disparity]\label{def:delta}
\[
\delta_{x,i} = \operatorname{mean}(x_{j,i}|s_j=0) - \operatorname{mean}(x_{j,i}|s_j=1)
\]
for $i\in[F_0]$. We use $\bm{\delta}_x$ to denote the vector of $\delta_{x,i}$ of size $F_0$, and use $|\delta_{x,i}|$ to denote its absolute value. 

\end{definition}
This is the difference of the mean of each feature between the two subgroups. Similarly, we can define the disparity of values of the $\ell$-th layer of DNN as 
\begin{equation}
    \delta_{h,i}^\ell = \operatorname{mean}(h_{j,i}^\ell|s_j=0) - \operatorname{mean}(h_{j,i}^\ell|s_j=1)
\end{equation}
for $i\in[F_\ell]$. At the output layer of the DNN, $\delta_{\hat{y}}=\delta_h^m$, which is a single value as the DNN is a binary classification. As shown in~\cite{kose2024fairgat}, $|\delta_{\hat{y}}| = \Delta_{SP}$ is exactly the statistical parity of the model. While the ensuing analysis will depend on $\Delta_{SP}$, the analysis can be readily extended to equal opportunity $\Delta_{EO}$, by also conditioning on the true label $y$.

In practice, the two aggregated information can be obtained from the national statistics released by the Census Bureau. For example,
publicly available statistical distributional information of different attributes such as income,  and expenses can be readily obtained from the Census Bureau website~\cite{census}. Alternatively, it can also be computed from a dataset possibly with a ZKP as well. In this scenario, the difference from the naive solution using ZK inferences is that our ZKP for fairness does not use the ZK inference as a subroutine, which leads to better efficiency and stronger security guarantees. See Section~\ref{subsec:setting} for more details about the setting and the threat model. \new{Moreover, to further protect the privacy of the dataset, an important advantage of our approach is that privacy enhancing techniques such as differential privacy, secure multiparty computations and homomorphic encryptions can be applied efficiently. This is because the information is simply aggregated statistics of the dataset, and does not involve any ML computations at all. This is left as an interesting future work.}

\subsection{Warm-up Example: Logistic Regression}\label{subsec:logistic}

The goal is to bound $\Delta_{SP}=|\delta_{\hat{y}}|$ only with the knowledge of $\Delta_{x}$ and $\bm{\delta}_{x}$, but not of any dataset. We start with the simple logistic regression model as a warm-up example. 

\begin{lemma}\label{lemma:sigmoid}
\begin{equation}
|\delta_{\hat{y}}| \leq L |\bar{z}^{(0)} -\bar{z}^{(1)}| + 2L\Delta_z,
\end{equation}
where $L$ is the Lipschitz constant of the sigmoid function, and $\bar{z}^{(0)} = \operatorname{mean}(z_j|s_j=0)$, $\bar{z}^{(1)} = \operatorname{mean}(z_j|s_j=1)$. 

\end{lemma}

\begin{proof} By the definition of $\delta_{\hat{y}}$ and $\hat{y}$,
\begin{equation}\label{eq:defs}
\begin{split}
    |\delta_{\hat{y}}|&:=\left|\operatorname{mean}(\hat{y}_j \mid s_{j}=0) - \operatorname{mean}(\hat{y}_j \mid s_{j}=1)\right|\\
    &= \left|\frac{1}{|\mathcal{S}_{0}|} \sum_{v_j \in \mathcal{S}_{0}} \hat{y}_j-\frac{1}{|\mathcal{S}_{1}|} \sum_{v_j \in \mathcal{S}_{1}} \hat{y}_j\right|\\
    &= \left|\frac{1}{|\mathcal{S}_{0}|} \sum_{v_j \in \mathcal{S}_{0}} \sigma(z_j)-\frac{1}{|\mathcal{S}_{1}|} \sum_{v_j \in \mathcal{S}_{1}} \sigma(z_j)\right|\\
\end{split}
\end{equation}

We can write $z_{j}= \bar{z}^{(s)} + \delta^{(s)}_{j}$, $\forall v_j \in \mathcal{S}_{s}$ , where $\bar{z}^{(s)} = \frac{1}{|\mathcal{S}_{s}|} \sum_{v_j \in \mathcal{S}_{s}} z_{j}$ for $s=0,1$. By Equation~\ref{eq:lipschitz},
\begin{equation}
\label{eq:main_ineq}
\begin{split}
\operatorname{\sigma}(\bar{z}^{(0)}) - L|\delta^{(0)}_{j}|  &\leq \operatorname{\sigma}(z_{j})= \operatorname{\sigma}((\bar{z}^{(0)}) + \delta^{(0)}_{j})\\
&\leq \operatorname{\sigma}((\bar{z}^{(0)})) + L|\delta^{(0)}_{j}|, \forall v_j \in \mathcal{S}_{0} \\
  \end{split}
\end{equation}
and 
\begin{equation}
\label{eq:main_ineq2}
\begin{split}
\operatorname{\sigma}(\bar{z}^{(1)}) - L|\delta^{(1)}_{j}|  &\leq \operatorname{\sigma}(z_{j})= \operatorname{\sigma}((\bar{z}^{(1)}) + \delta^{(1)}_{j})\\
&\leq \operatorname{\sigma}((\bar{z}^{(1)})) + L|\delta^{(1)}_{j}|, \forall v_j \in \mathcal{S}_{1} \\
  \end{split}
\end{equation}

Based on Equations \eqref{eq:main_ineq}, and \eqref{eq:main_ineq2}, we have:
\begin{equation}
\begin{split}
\frac{1}{|\mathcal{S}_{0}|} &\sum_{v_j \in \mathcal{S}_{0}} \left(\operatorname{\sigma}(\bar{z}^{(0)}) - L|\delta^{(0)}_{j}|\right ) - \\
&\frac{1}{|\mathcal{S}_{1}|} \sum_{v_j \in \mathcal{S}_{1}}  \left(\operatorname{\sigma}(\bar{z}^{(1)}) + L|\delta^{(1)}_{j}| \right)\\
\leq \frac{1}{|\mathcal{S}_{0}|} &\sum_{v_j \in \mathcal{S}_{0}} \sigma(z_j)-\frac{1}{|\mathcal{S}_{1}|} \sum_{v_j \in \mathcal{S}_{1}} \sigma(z_j)\\
\leq \frac{1}{|\mathcal{S}_{0}|} &\sum_{v_j \in \mathcal{S}_{0}} \left( \operatorname{\sigma}(\bar{z}^{(0)}) + L|\delta^{(0)}_{j}|\right) - \\
&\frac{1}{|\mathcal{S}_{1}|} \sum_{v_j\in \mathcal{S}_{1}}  \left(\operatorname{\sigma}(\bar{z}^{(1)}) - L|\delta^{(1)}_{j}|\right)
\end{split}
\end{equation}
Therefore, 
\begin{equation}
\label{eq:before_norm}
\begin{split}
&\operatorname{\sigma}(\bar{z}^{(0)}) - \operatorname{\sigma}(\bar{z}^{(1)})\\
&\qquad - L(\frac{1}{|\mathcal{S}_{0}|} \sum_{v_j \in \mathcal{S}_{0}} |\delta^{(0)}_{j}| - \frac{1}{|\mathcal{S}_{1}|} \sum_{v_j \in \mathcal{S}_{1}}  |\delta^{(1)}_{j}|)\\
\leq &\frac{1}{|\mathcal{S}_{0}|} \sum_{v_j \in \mathcal{S}_{0}} \sigma(z_j)-\frac{1}{|\mathcal{S}_{1}|} \sum_{v_j \in \mathcal{S}_{1}} \sigma(z_j)\\
\leq &\operatorname{\sigma}(\bar{z}^{(0)}) - \operatorname{\sigma}(\bar{z}^{(1)})\\
&\qquad+ L(\frac{1}{|\mathcal{S}_{0}|} \sum_{v_j \in \mathcal{S}_{0}} |\delta^{(0)}_{j}| + \frac{1}{|\mathcal{S}_{1}|} \sum_{v_j \in \mathcal{S}_{1}}  |\delta^{(1)}_{j}|)
\end{split}
\end{equation}
Therefore, Equation~\ref{eq:defs} is bounded by
\begin{equation}
    \begin{split}
        &\hspace{5mm}\left|\frac{1}{|\mathcal{S}_{0}|} \sum_{v_j \in \mathcal{S}_{0}} \sigma(z_j)-\frac{1}{|\mathcal{S}_{1}|} \sum_{v_j \in \mathcal{S}_{1}} \sigma(z_j)\right|\\
        &\leq |\operatorname{\sigma}(\bar{z}^{(0)}) - \operatorname{\sigma}(\bar{z}^{(1)})| \\
        &\qquad+ L(\big|\frac{1}{|\mathcal{S}_{0}|} \sum_{v_j \in \mathcal{S}_{0}} |\delta^{(0)}_{j}|\big| +\big|\frac{1}{|\mathcal{S}_{1}|} \sum_{v_j \in \mathcal{S}_{1}}  |\delta^{(1)}_{j}|\big|) \\
        &\leq |\operatorname{\sigma}(\bar{z}^{(0)}) - \operatorname{\sigma}(\bar{z}^{(1)})|+2L\Delta_z\\
        &\leq L |\bar{z}^{(0)} -\bar{z}^{(1)}| + 2L\Delta_z.
    \end{split}
\end{equation}

\end{proof}

To bound $|\bar{z}^{(0)} -\bar{z}^{(1)}|$, we have the following lemma:
\begin{lemma}\label{lemma:linear}
$|\bar{z}^{(0)} -\bar{z}^{(1)}|=\left|\inner{\bm{w},\bm{\delta_x}}\right|$. 
    
\end{lemma}
\begin{proof}
    This can be seen directly from the definition of $z$. 
\begin{equation}
    \begin{split}
        &\hspace{5mm}|\bar{z}^{(0)} -\bar{z}^{(1)}|\\
        &=\left|\frac{1}{|\mathcal{S}_{0}|} \sum_{v_j \in \mathcal{S}_{0}} z_{j}-\frac{1}{|\mathcal{S}_{1}|} \sum_{v_j \in \mathcal{S}_{1}} z_{j}\right|\\
        &=\left|\frac{1}{|\mathcal{S}_{0}|} \sum_{v_j \in \mathcal{S}_{0}} \inner{\bm{w},\bm{x}_j}-\frac{1}{|\mathcal{S}_{1}|} \sum_{v_j \in \mathcal{S}_{1}} \inner{\bm{w},\bm{x}_j}\right|\\
        &=\left|\inner{\bm{w}, \frac{1}{|\mathcal{S}_{0}|} \sum_{v_j \in \mathcal{S}_{0}} \bm{x}_j-\frac{1}{|\mathcal{S}_{1}|} \sum_{v_j \in \mathcal{S}_{1}}\bm{x}_j}\right|\\
        &=\left|\inner{\bm{w},\bm{\delta_x}}\right|
    \end{split}
\end{equation}
\end{proof}

To bound $\Delta_z$, we have:

\begin{lemma}
    \label{lemma:Deltaz}
    $\Delta_{z} = \inner{|\bm{w}|,\Delta_x}$.
\end{lemma}

\begin{proof}
    By the definition of $\Delta_z$,
\begin{equation}
\begin{split}
    &| z_{j}^{\ell} - \frac{1}{|\mathcal{S}_{0}|} \sum_{v_{j} \in \mathcal{S}_{0}} z_{j}^{\ell}| \\
    &= \left|\inner{\bm{w},\bm{x}_j}  - \frac{1}{|\mathcal{S}_{0}|} \sum_{v_{j} \in \mathcal{S}_{0}} \inner{\bm{w},\bm{x}_j}\right|\\
    &= \left|\inner{\bm{w}, (\bm{x}_j - \frac{1}{|\mathcal{S}_{0}|} \sum_{v_{j} \in \mathcal{S}_{0}} \bm{x}_j)}\right|\\
    &\leq\inner{|\bm{w}|,\Delta^{(0)}_x},
\end{split}
\end{equation}
where the last step follows the triangle inequality. Therefore, by Definition~\ref{def:Delta}, $\Delta_z^{(0)} = \inner{|\bm{w}|,\Delta^{(0)}_x}$. Similarly, we can show that $\Delta_z^{(1)}= \inner{|\bm{w}|,\Delta^{(1)}_x}$. Therefore, $\Delta_z = \inner{|\bm{w}|,\Delta_x}$.

\end{proof}

Combining Lemma~\ref{lemma:sigmoid},~\ref{lemma:linear} and~\ref{lemma:Deltaz}, for an LR model,
\begin{equation}\label{eq:LRfair}
    \delta_{\hat{y}}\leq L\left|\inner{\bm{w},\bm{\delta_x}}\right|+2L\inner{|\bm{w}|,\Delta_x}.
\end{equation}
The right-hand-side of the equation above is our fairness score for a logistic regression model. 

\subsection{DNN Fairness}\label{subsec:dnn}

To bound the disparity of a DNN model, we derive similar lemmas as Lemma~\ref{lemma:sigmoid},~\ref{lemma:linear} and~\ref{lemma:Deltaz} for each layer of the model. 

\begin{lemma}\label{lemma:nonlinear}
\begin{equation}
   \norm{\bm{\delta}_h^\ell}_2 \le L \norm{\bar{\bm{z}}^{\ell,(0)} -\bar{\bm{z}}^{\ell,(1)}}_2+2L\norm{\Delta_z^\ell}_2 \, ,
\end{equation}
where $\norm{\cdot}_2$ denotes the L2 norm of a vector. 
    
\end{lemma}

The proof is similar to Lemma~\ref{lemma:sigmoid} and we present it in Appendix~\ref{app:proof}. Essentially, the bound in Lemma~\ref{lemma:sigmoid} is true for $\delta_{h,i}^\ell$ of each dimension $i\in[F_\ell]$. Then Lemma~\ref{lemma:nonlinear} is true by taking the L2 norm of vectors because of the Cauchy-Schwarz Inequality. 

\begin{lemma}\label{lemma:linear2}
$\norm{\bar{\bm{z}}^{\ell,(0)} -\bar{\bm{z}}^{\ell,(1)}}_2 \le \norm{\mat{W}^{\ell-1}}_2  \norm{\bm{\delta}_h^{\ell-1}}_2$,
    
\end{lemma}
where $\norm{\mat{W}^{\ell-1}}_2$ denotes the spectral norm of the matrix $\mat{W}^{\ell-1}$. 

\begin{proof}
    By the definition of $z$,  
\begin{equation}
    \begin{split}
        &\norm{\bar{\bm{z}}^{\ell,(0)} -\bar{\bm{z}}^{\ell,(1)}}_2\\
        &=\norm{\frac{1}{|\mathcal{S}_{0}|} \sum_{v_j \in \mathcal{S}_{0}} \bm{z}^\ell_{j}-\frac{1}{|\mathcal{S}_{1}|} \sum_{v_j \in \mathcal{S}_{1}} \bm{z}^\ell_{j}}_2\\
        &=\norm{\frac{1}{|\mathcal{S}_{0}|} \sum_{v_j \in \mathcal{S}_{0}} \mat{W}^{\ell-1} \cdot \bm{h}_j^{\ell-1}-\frac{1}{|\mathcal{S}_{1}|} \sum_{v_j \in \mathcal{S}_{1}} \mat{W}^{\ell-1} \cdot \bm{h}_j^{\ell-1}}_2\\
        &=\norm{\mat{W}^{\ell-1} \cdot( \frac{1}{|\mathcal{S}_{0}|} \sum_{v_j \in \mathcal{S}_{0}} \bm{h}_j^{\ell-1}-\frac{1}{|\mathcal{S}_{1}|} \sum_{v_j \in \mathcal{S}_{1}}\bm{h}_j^{\ell-1})}_2\\
        &=\norm{\mat{W}^{\ell-1}\cdot \bm{\delta}_h^{\ell-1}}_2\\
        &\le \norm{\mat{W}^{\ell-1}}_2  \norm{\bm{\delta}_h^{\ell-1}}_2. 
    \end{split}
\end{equation}
The last step of the inequality follows the definition of the spectral norm of a matrix.

\end{proof}

Again, all the steps of the proof except for the last one are similar to the proof of Lemma~\ref{lemma:linear}, adapted to the matrix of $\mat{W}^{\ell-1}$. The reason for the last step is that unlike the logistic regression, we need to continue bounding the disparity layer by layer for the DNN model. Therefore, we would like to reduce the bound to have the term $\norm{\bm{\delta}_h^{\ell-1}}_2$ as in Lemma~\ref{lemma:linear2}, so that we can apply Lemma~\ref{lemma:nonlinear} again for $\ell-1$. 

\begin{lemma}\label{lemma:Deltaz2}
$\Delta_z^\ell \le L|\mat{W}^{\ell-1}|\cdot\Delta_z^{\ell-1}$, and $\Delta_z^1 \le |\mat{W}^{0}|\cdot  \Delta_x$, where $|\mat{W}|$ denotes taking the absolute value of each element in $\mat{W}$. 
\end{lemma}

\begin{proof}
By the definition of $\Delta_{z,i}^\ell$, for each dimension $i\in[F_\ell]$, let $\mat{W}_i^{\ell-1}$ be the $i$-th row of matrix $\mat{W}^{\ell-1}$, 
\begin{equation}
\begin{split}
    &| z_{j,i}^{\ell} - \frac{1}{|\mathcal{S}_{0}|} \sum_{v_{j} \in \mathcal{S}_{0}} z_{j,i}^{\ell}|\\
    &= \left| \inner{\mathbf{W}^{\ell-1}_i,  \sigma(\bm{z}^{\ell-1}_j)} - \frac{1}{|\mathcal{S}_{0}|} \sum_{v_{j} \in \mathcal{S}_{0}} \inner{\mathbf{W}^{\ell-1}_i, \sigma(\bm{z}^{\ell-1}_j)}\right|\\
    &= \left|\inner{ \mathbf{W}^{\ell-1}_i, (\sigma(\bm{z}^{\ell-1}_j) - \frac{1}{|\mathcal{S}_{0}|} \sum_{v_{j} \in \mathcal{S}_{0}} \sigma(\bm{z}^{\ell-1}_j))}\right|\\
    &\leq \sum_{k=0}^{F_{\ell-1}-1} |\mathbf{W}^{\ell-1}_{i,k}| \left|\sigma(z^{\ell-1}_{j,k}) - \frac{1}{|\mathcal{S}_{0}|} \sum_{v_{j} \in \mathcal{S}_{0}} \sigma(z^{\ell-1}_{j,k})\right|\\
   &\leq \sum_{k=0}^{F_{\ell-1}-1} |\mathbf{W}^{\ell-1}_{i,k}| \left|L(z^{\ell-1}_{j,k} - \frac{1}{|\mathcal{S}_{0}|} \sum_{v_{j} \in \mathcal{S}_{0}} z^{\ell-1}_{j,k})\right| \\
    &= L \sum_{k=0}^{F_{\ell-1}-1} |\mathbf{W}^{\ell-1}_{i,k}|\Delta_{z,k}^{\ell-1, (0)}
\end{split}
\end{equation}
Therefore, for each $i$, $\Delta_{z,i}^{\ell, (0)} \le L \sum_{k=0}^{F_{\ell-1}-1} |\mathbf{W}^{\ell-1}_{i,k}|\Delta_{z,k}^{\ell-1, (0)}$. As a vector, $\Delta_{z}^{\ell, (0)} \le L|\mat{W}^{\ell-1}|\cdot \Delta_z^{\ell-1, (0)}$. Similarly, $\Delta_{z}^{\ell, (1)} = L|\mat{W}^{\ell-1}|\cdot \Delta_z^{\ell-1, (1)}$, and thus $\Delta_{z}^{\ell} \le L|\mat{W}^\ell|\cdot \Delta_z^{\ell-1}$. 

As a special case, when $\ell = 1$, $\bm{z}^1 = \mat{W}^0\cdot \bm{x}$, and thus $\Delta_z^1 \le |\mat{W}^{0}|\cdot \Delta_x$ due to a similar proof without the activation function and the Lipschitz constant. 

\end{proof}

Combining Lemma~\ref{lemma:nonlinear},~\ref{lemma:linear2} and~\ref{lemma:Deltaz2}, the disparity of the DNN model can be bounded recursively as:
\begin{equation}\label{eq:DNNfair}
    \begin{split}
        &\norm{\bm{\delta}_h^\ell}_2 \le L \norm{\mat{W}^{\ell-1}}_2  \norm{\bm{\delta}_h^{\ell-1}}_2+2L\norm{\Delta_z^\ell}_2, \\
        & \Delta_z^\ell \le L|\mat{W}^{\ell-1}|\cdot\Delta_z^{\ell-1},
    \end{split}
\end{equation}
with $\bm{\delta}_h^0=\bm{\delta}_x$ and $\Delta_z^1 \le |\mat{W}^{0}|\cdot\Delta_x$. The bound on the disparity of the output layer is our fairness score for the DNN model. That is, $|\delta_{\hat{y}}|=|\delta_h^m| = \norm{\delta_h^m}_2$, as it is a single value with dimension 1 as the binary classification. 

\subsection{Improvements over~\cite{kose2024fairgat}}\label{subsec:diff}

In~\cite{kose2024fairgat}, Kose and Shen analyzed the bounds of disparity for Graph Attention Networks. A bound for the fully connected layer was given in~\cite[Equation 5]{kose2024fairgat}, and our analysis is inspired by this work. Here we summarize our major improvements over~\cite{kose2024fairgat}:
\begin{enumerate}[leftmargin=*]
    \item \textbf{Bounding $\Delta_z$.} The most important change is the new bounds for $\Delta_z$ in Lemma~\ref{lemma:Deltaz} and~\ref{lemma:Deltaz2}. The bound in~\cite{kose2024fairgat} depends on $\Delta_z^\ell$, the bounded value of the intermediate result of each layer. This was fine in~\cite{kose2024fairgat} for the purpose of training a fair ML model with access to the training dataset and all intermediate values. However, in our paper, we aim to build a ZKP to prove the fairness of the model without access to any specific dataset, thus the prover does not have access to $\Delta_z^\ell$. Therefore, we derive new equations to bound $\Delta_z^\ell$ layer by layer with the model $\mat{W}$, and eventually reduce it to the input layer with $\Delta_x$. 
    \item \textbf{Tighter bounds for DNN.} Our bound in Lemma~\ref{lemma:nonlinear} is tighter than that in~\cite{kose2024fairgat}. In~\cite[Equation 5]{kose2024fairgat}, there is a multiplicative factor of $\sqrt{F_\ell}$ in front of $\Delta_z$. \footnote{It was a typo in~\cite{kose2024fairgat} writing it as $\sqrt{N}$, the number of data samples. It should be $\sqrt{F}$, the dimension of the features. } $\sqrt{F_\ell}$ can be large in DNN, e.g., $\sqrt{F_\ell}=11.3$ for a layer with 128 neurons, making the bound very loose. In our bound, we eliminate this term by doing a fine-grained analysis on each dimension of $\Delta_{z,i}^\ell$ in Lemma~\ref{lemma:Deltaz2}. Our bound is strictly better than that in~\cite{kose2024fairgat}.  
     \item \textbf{Tighter bounds for logistic regression.} We propose a better bound for the logistic regression in Lemma~\ref{lemma:sigmoid} and~\ref{lemma:linear}. The equation in Lemma~\ref{lemma:linear} is strictly equal instead of an upper bound. 
\end{enumerate}

As we will show in Section~\ref{sec:exp}, our bounds are better than those in~\cite{kose2024fairgat} and better reflect the fairness of the ML models.

\section{Efficient zkSNARKs for ML Fairness}\label{sec:zk}

In this section, we present the design of our zkSNARKs for ML fairness. 

\subsection{Settings and Threat Model}\label{subsec:setting}

As mentioned in Figure~\ref{fig:design}, we consider a scenario where the model owner has a closed-source ML model and is not willing to share the model with the public. The model owner is not trusted, and thus needs to convince the client that the model is fair under the definitions of group fairness, without revealing the sensitive parameters of the model. The client has access to the commitment of the model provided by the model owner, which can be used to ensure that the same fair model is used in other applications later. The client is not trusted for the confidentiality of the model, and thus the model should remain private to the model owner. 

We consider two settings in our schemes. In the first setting, we assume that the aggregated statistics $\Delta_x, \bm{\delta}_x$ as defined in Section~\ref{subsec:delta} are publicly known by the client. As mentioned earlier, they can be computed from the reports released by public agencies. In this setting, the relationship of the zkSNARK to prove is:
\begin{align*}
  \mathcal{R}_\mathsf{fair} &= \{\left( (vk, \ms{com}_\mat{W},\Delta_x, \bm{\delta}_x, \delta_{\hat{y}} ); (\mat{W}) \right):\\
  &\ms{com}_{\mat{W}}= \ms{Commit}(\mat{W},r) \land \ms{FairEval}(\mat{W},\Delta_x, \bm{\delta}_x) = \delta_{\hat{y}}\}  
\end{align*}
Here $\mat{W}$ denotes the model parameters and $\ms{FairEval}$ denotes the bounds of the fairness scores in Section~\ref{sec:fairness}. 

In the second setting, when the aggregated statistics are not publicly available, they can be computed from a dataset possibly owned by a third party. In this setting, the dataset should remain confidential as it may contain sensitive information of people. Moreover, keeping the dataset secret is beneficial to the fairness testing, as the model owner cannot access the testing dataset ahead of time and make the model fair only for the specific dataset, and the dataset can be reused for testing multiple models. The data owner is also not trusted, and should convince the client that the aggregated statistics are computed correctly. The zkSNARK relationship is:
\begin{align*}
  \mathcal{R}_\mathsf{stat} &= \{\left( (vk, \ms{com}_\mat{D},\Delta_x, \bm{\delta}_x ); (\mat{D}) \right):\\
  &\ms{Open}(\ms{com}_{\mat{D}}=1,\mat{D},r) \land \ms{Aggregate}(\mat{D}) = (\Delta_x, \bm{\delta}_x)\}  
\end{align*}
where $\mat{D}$ denotes the dataset and $\ms{Aggregate}$ computes $\Delta_x, \bm{\delta}_x$ defined in Section~\ref{subsec:delta}. Again, this part is computed separately from the fairness proof of the model. Combining the two zkSNARKs together, our solution provides both better efficiency and stronger security guarantees in this setting. 

\subsection{Optimized zkSNARK Gadgets}\label{subsec:zkgadget}

We first present the optimized zkSNARK building blocks we developed for common computations used in proving the fairness of ML models in Section~\ref{sec:fairness}. 

\subsubsection{Quantization, Absolute Value and Maximum}\label{subsubsec:range} 

The first type of gadgets heavily utilizes range checks with lookup arguments, which has become a common approach in recent efficient zkSNARK schemes. 

\paragraph{Encoding real numbers with quantization.} Our zkSNARK scheme works on a finite field $\F_p$ defined by a large prime $p$, while the data and the model parameters are represented as real numbers. We convert them to integers in $\F_p$ in a straight-forward way through quantization and simulate fix-point arithmetic in zkSNARK. This approach leads to better efficiency, and our fairness computations do not require very high precision with floating-point numbers and arithmetic. 

A real number is quantized to $q< \floor{\log p}-1$ bits in total, with $q_I$ bits for the integer part and $q_D = q-q_I$ bits for the decimal part. A positive number is then encoded as an integer in $[0, 2^q - 1]$ in $\F_p$ by multiplying it with the scalar $2^{q_D}$. A negative number is encoded as an integer in $[p-(2^q-1), p-1]$ by multiplying the scalar $2^{q_D}$ and converting the result to $\F_p$. 

\new{To check the validity of numbers committed by the prover, such as the data and the model parameters, for each number $v\in V$, the prover additionally commits to its sign $\ms{sign}_v = 1$ or $p-1$. A public lookup table $T=(0,1,\ldots, 2^q-1)$ is then constructed, and the prover proves that $\ms{Quantize\_Validate}(v):$
\begin{enumerate}[leftmargin=*]
    \item $\ms{sign}_v\cdot \ms{sign}_v = 1$ for all $v\in V$;
    \item $(\ms{sign}_v \cdot v)_{v\in V} =  \ms{Lookup}(T)$. 
\end{enumerate}
The first check can be implemented by the GKR protocol in~\ref{prot:gkr} with multiplication gates and the second check is a lookup argument followed by multiplications in GKR. }

\paragraph{Absolute Value.} An advantage of committing to the sign bit of values above is that computing the absolute values becomes simple. That is, $|v| = \ms{sign}\cdot v$ without any additional checks. It is a common operation applied to every element of $\bm{\delta}$ and $\mat{W}$ in Lemma~\ref{lemma:linear},~\ref{lemma:Deltaz} and~\ref{lemma:Deltaz2}. This can again be implemented by GKR and we denote it as $\ms{Absolute}(v)$. 

\paragraph{Multiplications.} We simulate fixed-point arithmetic with truncations. We assume that the field $\F_p$ is big enough so that there is no overflow after a multiplication in $\F_p$. Let $c = a\times b$ be the result of multiplication between two quantized numbers in $\F_p$, which has $2q_D$ bits for the decimal part. The prover provides $\hat{c}$ and the truncation error $e$, and proves 
\begin{enumerate}[leftmargin=*]
    \item $c = \hat{c}\times 2^{q_D}+e$ via the GKR protocol;
    \item $\hat{c} =  \ms{Lookup}(T)$ and $e = \ms{Lookup}(T_{\ms{truncate}})$, 
\end{enumerate}
where $T_{\ms{truncate}}={0,1,\ldots, 2^{q_D}-1}$. When the field size is large enough, we can do the truncation after several multiplications as an optimization.

\paragraph{Maximum.} Another common operation used in computing $\Delta_x$ from the dataset is the maximum. It can again be supported efficiently by the range check. To compute the maximum value of $n$ numbers $\bm{a}=\{a_0,a_1,\ldots, a_{n-1}\}$ (already quantized and validated in $\F_p$), the prover commits to $a_{\max}$ as an auxiliary input, and proves $\ms{Max}(\bm{a})$:
\begin{enumerate}
    \item $a_{\max}-a_i$ is non-negative: $((a_{\max}-a_i)_{i\in[n]}) = \ms{Lookup} (T)$, where $T=(0,1,\ldots, 2^q-1)$;
    \item $a_{\max}$ equals to one of the elements: $\prod_{i\in[n]}(a_{\max}-a_i) = 0$, using a multiplication tree in GKR.  
\end{enumerate}

\paragraph{Complexity.} All of the checks above are reduced to the GKR-based zkSNARK with lookup arguments of size linear in the number of elements. With the protocols in~\cite{t13,xielibra,habock2022multivariate}, the prover time is $O(n+|T|)$ where $n$ is the number of operations (quantization, truncation and maximum) and $|T| = 2^q$ is the size of the table. With the proper choice of $q$, this approach is more efficient than the bit-decomposition method with a multiplicative overhead of $O(q)$.

\subsubsection{Spectral Norm}\label{subsubsec:spectral}

Computing the spectral norm of a matrix as in Lemma~\ref{lemma:linear2} is the most challenging operation to prove in our scheme. As described in Section~\ref{subsec:mlprelim}, the spectral norm of a matrix $\mat{W}$ is the square root of the largest eigenvalue of $\mat{A} = \mat{W}\cdot \mat{W}^T$. There are standard algorithms to compute the spectral norm such as eigenvalue decomposition and power iteration \cite{stoll2013linear}, but the overhead to implement these algorithms in zkSNARK naively would be high. For example, naive eigenvalue decomposition entails a complexity of $\mathcal{O}(F^3)$, where $F$ is the number of rows of the matrix.

\begin{figure*}[t!]
   \centering
    \begin{minipage}{\textwidth}

   \begin{algorithm}[H]
 \new{  
\raggedright
	\textbf{Public input:} polynomial commitment $\ms{com}_W$ of the MLE of $\mat{W}$, lookup tables $T_{\ms{err}}, T_{2^q}$\\
        \textbf{Witness:} $\mat{W}$\\
        \textbf{Auxiliary input:} eigenvalues $\bm{\lambda}$, eigenvectors $\mat{V}$, maximum eigenvalue $\lambda_\ms{max}$, errors $\mat{E},\mat{E}',e$\\
	\textbf{Output:} spectral norm $\norm{\mat{W}}_2$ 
	\begin{algorithmic}[1]		
        \State $\P$ commits to MLEs $V,\lambda,E,E'$ defined by matrices and vectors $\mat{V},\bm{\lambda},\mat{E},\mat{E'}$, as well as $\lambda_\ms{max}, e$, using the polynomial commitment scheme $\pcs.\Commit$. 
        \State $\P$ and $\V$ executes lookup argument $(\mat{E},\mat{E'},e)=\ms{Lookup}(T_{\ms{err}})$, where the elements of $\mat{E},\mat{E'},e$ are viewed as a vector. 
         \State $\P$ and $\V$ executes lookup argument $(\mat{V})=\ms{Lookup}(T_{2^q})$
         \State $\P$ and $\V$ executes lookup argument $(\inner{\bm{v}_i,\bm{v}_i}-1)_{i\in[F]} = \ms{Lookup}(T_{2^q})$.
         
        \State $\V$ samples $\bm{r}_1,\bm{r}_2$ randomly, receives $E'(\bm{r}_1,\bm{r}_2)$ from $\P$, and executes the sumcheck protocol on $E'(\bm{r}_1,\bm{r}_2)+\tilde{eq}(\bm{r}_1,\bm{r}_2) = \sum_{\bm{x}\in\binary^{\log F}} V(\bm{r}_1,\bm{x})\cdot V(\bm{r}_2,\bm{x})$. 
        
         
        \State $\V$ samples $\bm{r}_3,\bm{r}_4$ randomly, receives $E(\bm{r}_3,\bm{r}_4)$ from $\P$, and executes the sumcheck protocol on $E(\bm{r}_3,\bm{r}_4) = \sum_{\bm{x}\in\binary^{\log F'}} W(\bm{x},\bm{r}_3)\cdot W(\bm{x},\bm{r}_4) - \sum_{\bm{x}\in\binary^{\log F}} V(\bm{r}_3,\bm{x})\cdot\lambda(\bm{x})\cdot V(\bm{r}_4,\bm{x})$

        
        \State $\lambda_{\max} = \ms{Max}(\bm{\lambda})$.
        \State $\norm{\mat{W}}^2_2 = \lambda_{\max}+e$.
        \State $\P$ opens the polynomials $W,V,\lambda, E,E'$ at random points queried by the sumcheck protocol and the lookup arguments using $\pcs.\Open$, and $\V$ verifiers their correctness by $\pcs.\Verify$. 
	\end{algorithmic}

}
    
	\end{algorithm}

 \end{minipage}

    \caption{Protocol for proving the spectral norm of a matrix $\mat{W}$. }
    \label{fig:spectral}
\end{figure*}

\paragraph{Verifying eigenvalues.} A natural approach is to verify the result instead of directly computing it in zkSNARK. We take this approach as our starting point. The prover computes $\mat{A} = \mat{W}\cdot \mat{W}^T$ outside the zkSNARK, and its validity can be checked efficiently via a sumcheck protocol. For simplicity, we drop the superscript and denote the dimension of $\mat{W}$ as $F\times F'$ and $\mat{A}$ as $F\times F$. The prover then commits to all eigenvalues $\bm{\lambda} = (\lambda_0,\ldots, \lambda_{F-1})$ of $\mat{A}$ and their corresponding eigenvectors $\mat{V} = (\bm{v}_0, \ldots, \bm{v}_{F-1})$. By the property of eigenvalues, the zkSNARK checks that $\mat{A}\cdot \bm{v}_i = \lambda_i\bm{v}_i$ and $\bm{v}_i\neq \bm{0}$ for all $i\in[F]$. In the matrix form, the equation is $\mat{A} = \mat{V}\cdot \mat{\Lambda}\cdot \mat{V}^T$, where eigenvalues are interpreted as a diagonal matrix $\mat{\Lambda}$. Finally, the zkSNARK selects the maximum eigenvalue from $\bm{\lambda}$ using the maximum gadget in the previous subsection and computes its square root.

\paragraph{Introducing error terms.} However, this approach does not work directly in zkSNARK. The reason is because of the encoding of numbers in $\F_p$ via quantization. The eigenvalues and eigenvectors are computed over the real numbers. After quantization, there may not exist an eigenvector $\bm{v}_i$ in our valid range of numbers in $\F_p$ that satisfies $\mat{A}\cdot \bm{v}_i = \lambda_i\bm{v}_i$. Note that this problem is not unique to our encoding. Even with floating-point numbers, the precision may not be enough to represent the eigenvalues and eigenvectors and the equation is an approximate equality in floating-point arithmetic as well. 

To solve this problem, we introduce an error term $\bm{e}_i$ for each eigenvector. The prover commits to $\bm{e}_i\in\F_p^F$ and proves that (1) $\mat{A}\cdot \bm{v}_i = \lambda_i\bm{v}_i+\bm{e}_i$, (2) $\bm{e}_i$ is small via a lookup argument. Writing them as matrices, the prover commits to $\mat{V},\mat{E}\in\F_p^{F\times F}$ and the diagonal matrix $\mat{\Lambda}$, and proves $\mat{A} = \mat{V}\cdot \mat{\Lambda}\cdot \mat{V}^T+\mat{E}$, and $\mat{E} = \ms{Lookup}(T_\ms{err})$.  

Note that robustness of the eigenvalue computation is not an issue here and the prover can always find $\lambda_i,\bm{v}_i$ with a small error. The prover computes the eigenvalues and eigenvectors of $\mat{A}$ over the real numbers externally outside of the zkSNARK. Then they are converted to $\F_p$ using the quantization. As additions and multiplications are robust (small errors on the input leads to small errors of the output), the error term after quantization is small.

\paragraph{Checking unique eigenvalues.} Another major issue of verifying the eigenvalues is that a malicious prover may not provide all eigenvalues of $\mat{A}$. The prover could use the same $\lambda_i, \bm{v}_i$ for all $i\in[F]$ to pass the checks. By doing so, the maximum would return (the square root of) this eigenvalue as the spectral norm, instead of the maximum eigenvalue. Even worse, because of the error term above, the prover could use multiple values that are close to one eigenvalue and eigenvector, and figure out the corresponding error terms to pass the checks. Therefore, simply checking that the eigenvalues are distinct does not solve the problem. 

To address this issue, we propose another step to prove that all eigenvectors are orthogonal to each other. That is, $\inner{\bm{v}_i,\bm{v}_j}=0 \quad\forall i,j\in[F]$. This ensures that the eigenvector space is full rank, and thus is a sufficient condition that there is no other independent eigenvalue and eigenvector pair. In the matrix form, it is $\mat{V}\cdot \mat{V}^T = \mat{I}$, where $\mat{I}$ is the identity matrix. An honest prover can always find such pairwise orthogonal eigenvectors for $\mat{A}=\mat{W}\cdot\mat{W}^T$, a symmetric positive semi-definite matrix. In zkSNARK, again because of quantization, the inner product may not be exactly 0, and we again allow a small error term $\mat{E}'$. That is, the prover commits to $\mat{E}'\in\F_p^{F\times F}$ and proves $\mat{V}\cdot \mat{V}^T = \mat{I}+\mat{E}'$, and $\mat{E}' = \ms{Lookup}(T_\ms{err})$. In Appendix~\ref{app:error}, we give a formal proof that the eigenvalues provided by the prover are close to the true eigenvalues over the real numbers with these checks. 


\paragraph{Full protocol.} Our new protocol to prove the spectral norm of a matrix is presented in Figure~\ref{fig:spectral}. It is specified as PCS, lookup arguments, sumcheck and GKR between the prover and the verifier, but can be compiled to a zkSNARK using standard approaches described in Appendix~\ref{app:sumcheck}. 
The public input includes the commitment of $\mat{W}$ by committing to the MLE $W$ defined by $\mat{W}$ as in Definition~\ref{def:mle} via a polynomial commitment scheme. To check the range of errors and quantized values, we have two lookup tables $T_{\ms{err}}, T_{2^q}$. To achieve sublinear verifier time, the prover also commits to their MLEs and open them using the polynomial commitment scheme, but we list them as public input as they are known by the verifier. We separate the real witness of $\mat{W}$ from the auxiliary input $\bm{\lambda},\mat{V}, \mat{E},\mat{E'},e$ to improve the efficiency of the zkSNARK. 

In Step 1, the prover commits to the auxiliary input using the PCS scheme. Step 2 and 3, prove that the errors are small and other values are in the legitimate range of quantization. Step 4 checks that the eigenvectors committed by $\P$ are nonzero. It suffices to check that $\inner{\bm{v}_i,\bm{v}_i}-1 \in [0, 2^q-1]$ via a lookup argument.

Step 5 is the sumcheck protocol for $\mat{V}\cdot \mat{V}^T = \mat{I}+\mat{E}'$. By the definition of the identity matrix, its MLE evaluated at $\bm{r}_1,\bm{r}_2$ is exactly equal to $\tilde{eq}(\bm{r}_1,\bm{r}_2)$, which can be computed efficiently by the verifier. Step 6 checks that $E = \mat{W}\cdot\mat{W}^T - \mat{V}\cdot \mat{\Lambda}\cdot \mat{V}^T$, which combines both $\mat{A} = \mat{V}\cdot \mat{\Lambda}\cdot \mat{V}^T+\mat{E}$ and $\mat{A} = \mat{W}\cdot\mat{W}^T$ in a single sumcheck equation. This is to avoid the commitment to $\mat{A}$ in the auxiliary input as an optimization. The prover needs to compute $\mat{A}$ in order to compute the eigenvalues and eigenvectors, but this is completely outside the zkSNARK. As the eigenvalues are interpreted as a diagonal matrix $\mat{\Lambda}$, the values are nonzero if and only if $x = y$, thus the $\mat{V}\cdot\mat{\Lambda}\cdot\mat{V}^T$ can be expressed as $\sum_{x\in\binary^{\log F}} (\sum_{y\in\binary^{\log F}}V(\bm{r}_3,y)\cdot\tilde{eq}(x,y)\lambda(x)) \cdot V(\bm{r}_4,x)$, which can be simplified as $\sum_{x\in\binary^{\log F}} V(\bm{r}_3,x)\cdot\lambda(x)\cdot V(\bm{r}_4,x)$ as in the second term on the right hand side of the equation in Step 6.\footnote{In our implementation, due to the encoding of real numbers in $\F_p$, we multiply $2^{q_D}$ to the first term $W(\bm{x},\bm{r}_3)\cdot W(\bm{x},\bm{r}_4)$ to make the length of its decimal part the same as the second term. For the same reason, the lookup arguments on $E,E'$ are applied after truncation.}

Finally, Step 7 computes the maximum eigenvalue from $\bm{\lambda}$ using the protocol $\ms{Max}$ described in the previous section, and Step 8 verifies that $\norm{\mat{W}}_2$ is indeed the square root of the maximum eigenvalue. To do so, we check that $\norm{\mat{W}}^2_2=\lambda_{\max}+e$ with a small error $e$, which is implemented as addition and multiplication gates in the GKR protocol in Protocol~\ref{prot:gkr}.

At the end of the lookup arguments, sumcheck and GKR protocols, the verifier needs to evaluate $V(),\lambda(), E(), E'(), W()$ at random points, which are opened by the prover using the PCS. Standard techniques in prior work~\cite{xielibra,zhang2021doubly} can be used to combine multiple evaluations of a polynomial into one using another sumcheck protocol with minimal overhead. $e, \lambda_{\max}$ are also committed and opened by the prover, but we omit them to improve the readability of the protocol. Overall, our protocol can be viewed as an extension of the GKR-based zkSNARK with lookup arguments presented in Appendix~\ref{app:sumcheck}. Additional sumchecks in Step 5 and 6 are introduced in addition to addition and multiplication gates to improve the prover efficiency for proving the spectral norm. 

\paragraph{Complexity.} In total, there are 3 sumcheck protocols with the maximum degree 3 in each variable. With the algorithms in~\cite{xielibra}, the prover time is $O(F^2+F\cdot F')$. The total number of lookups is $O(F^2)$ to $T_{\ms{err}}$ and $O(F^2)$ to $T_{2^q}$. Other steps only introduce lower order terms in the complexity. The total prover time is $O(F^2+F\cdot F'+|T_{\ms{err}}+|T_{2^q}|)$. This is faster than the prover time of the naive approach computing the spectral norm in zkSNARK by a factor of $O(F)$. The proof size and the verifier time are both $O(\log F+\log F')$, excluding those incurred by the polynomial commitment scheme.

\subsection{zkSNARKs for ML Models}\label{subsec:zkfair} 

With the zkSNARK gadgets above, we present our protocols to prove the fairness of ML models in this section. 

\begin{figure}[t!]
   \centering
    \begin{minipage}{\linewidth}
   
   \begin{algorithm}[H]

\raggedright
	\textbf{Public input:} commitment $\ms{com}_{\bm{w}}$ of the logistic regression model $\bm{w}$, $\bm{\delta}_x$ and $\Delta_x$ \\
        \textbf{Witness:} the logistic regression model $\bm{w}$\\
	\textbf{Output:} $\delta_{\hat{y}}$ 
	\begin{algorithmic}[1]		
        \State $\ms{Quantize\_Validate}(\bm{w})$
        \State $|\bm{w}| = \ms{Absolute}(\bm{w})$
        \State $\delta_{\hat{y}} = L \cdot \ms{Absolute}(\inner{\bm{w},\bm{\delta}_x})+2L\cdot\inner{|\bm{w}|,\Delta_x}$ with $L = 0.25$
	\end{algorithmic}
	\end{algorithm}
 \end{minipage}

    \caption{zkSNARK for proving the fairness of a logistic regression model. }
    \label{fig:LRfair}
\end{figure}

\new{
\subsubsection{Logistic Regression Fairness}\label{subsubsec:zkLR}

The fairness score of a logistic regression model is given by Equation~\ref{eq:LRfair}. To prove the equation, our protocol is presented in Figure~\ref{fig:LRfair}. In the protocol, $\ms{Quantize\_Validate}$ and $\ms{Absolute}$ denote the protocols of checking the validity of quantized values and computing the absolute values described in Section~\ref{subsubsec:range}. Multiplications are performed with truncations in Section Section~\ref{subsubsec:range} as well. 

\begin{theorem}
    The protocol in Figure~\ref{fig:LRfair} is a zkSNARK under Definition~\ref{def:snark} for relation $\mathcal{R}_\ms{fair}$ for the logistic regression model $\bm{w}$.  
\end{theorem}

\paragraph{Complexity.} For a model of dimension $F$, the prover time is only $O(F)$ plus $O(F)$ lookups.

\begin{figure}[t!]
   \centering
    \begin{minipage}{\linewidth}
   
   \begin{algorithm}[H]

\raggedright
	\textbf{Public input:} commitment $\ms{com}_{\mat{W}}$ of the DNN model $\mat{W} = (\mat{W}^0, \ldots, \mat{W}^{m-1})$, $\bm{\delta}_x$ and $\Delta_x$ \\
        \textbf{Witness:} the DNN model $\mat{W}$\\
	\textbf{Output:} $\delta_{\hat{y}}$ 
	\begin{algorithmic}[1]		
        \State $\ms{Quantize\_Validate}(\mat{W})$ 
        \For{$\ell=0,1,\ldots, m-1$}
            \State $|\mat{W}^\ell| = \ms{Absolute}(\mat{W}^\ell)$
            \State $\norm{\mat{W}^\ell}_2 = \ms{Spectral}(\mat{W}^\ell)$
        \EndFor
        \State $\norm{\bm{\delta}_h^0}_2=\norm{\bm{\delta}_x}_2$ and $\Delta_z^1 = |\mat{W}^1|\cdot \Delta_x$. 
        \For{$\ell=1,\ldots, m$}
            \State $\norm{\bm{\delta}_h^\ell}_2 = L\norm{\mat{W}^{\ell-1}}_2\norm{\bm{\delta}_h^{\ell-1}}_2+2L\norm{\Delta_z^\ell}_2$
            \State $\Delta_z^{\ell+1} = L |\mat{W}^\ell| \cdot \Delta_z^\ell $ if $\ell\le m-1$ 
        \EndFor
        \State $\delta_{\hat{y}} = \norm{\bm{\delta}_h^m}_2$
	\end{algorithmic}
	\end{algorithm}
 \end{minipage}

    \caption{zkSNARK for proving the fairness of a DNN model. }
    \label{fig:DNNfair}
\end{figure}

\subsubsection{DNN Fairness}\label{subsubsec:zkDNN} 

The fairness score of a DNN model is given by Equation~\ref{eq:DNNfair}. The protocol is presented in Figure~\ref{fig:DNNfair}. In addition to $\ms{Quantize\_Validate}$ and $\ms{Absolute}$, $\ms{Spectral}$ denotes the protocol of proving the spectral norm of a matrix in Figure~\ref{fig:spectral}. In Step 5-8, the computations of matrix-vector multiplications and L2 norm are performed using the GKR protocols with truncations described in Section ~\ref{subsubsec:range}. 

\begin{theorem}
    The protocol in Figure~\ref{fig:DNNfair} is a zkSNARK under Definition~\ref{def:snark} for relation $\mathcal{R}_\ms{fair}$ for the DNN model $\mat{W}$.  
\end{theorem}

\paragraph{Complexity.} The prover time of our zkSNARK is $O(\sum_{\ell=0}^{m-1} F_\ell\cdot F_{\ell+1})$ with $O(\sum_{\ell=0}^{m-1} F_\ell\cdot F_{\ell+1})$ lookups. This is proportional to the size of the model $\mat{W}$ and is the same as the prover time of a single inference of the same model asymptotically. Concretely, the dominating cost is from proving the spectral norm of each weight matrix in the DNN. 

\paragraph{Leakage of the protocols.} With the zero-knowledge variant of the GKR-based zkSNARK and lookup arguments, the proof does not reveal any information about the model parameters and auxiliary input used in the protocols. In addition to the final fairness score, to define the relation of zkSNARK, the verifier does learn the architecture of the model, including the dimensions of the model matrices, the number of layers and the type of activation functions (to compute the Lipschitz constant). If the information is sensitive, the leakage can be further mitigated by padding the model architecture to public upperbounds with dummy neurons and layers.

}
\subsubsection{Aggregated Statistics}\label{subsubsec:delta} 

Finally, the zkSNARK to prove $\bm{\delta}_x,\Delta_x$ of a dataset is presented in Figure~\ref{fig:delta} in Appendix~\ref{app:delta}. They are computed as defined in Definition~\ref{def:Delta} and~\ref{def:delta} and the zkSNARK is used in the second setting where the aggregated statistics are not available and are computed from a private dataset. 

\begin{theorem}
    The protocol in Figure~\ref{fig:delta} is a zkSNARK under Definition~\ref{def:snark} for relation $\mathcal{R}_\ms{stat}$.  
\end{theorem}

\paragraph{Complexity.} The prover time is $O(N\cdot F)$ with $O(N\cdot F)$ lookups to validate the quantization and compute the maximum. This is proportional to the size of the dataset.

\begin{table}[b!]
    \centering
    \renewcommand{\arraystretch}{1.5}
    \begin{tabular}{l|c|c|c|c}
    \toprule
        Dataset & $N$ & $F$ & $\mathbf{s}$ & $\mathbf{y}$\\\midrule
        {Adult} & $45222$ & $38$ & gender & income level\\
         {German} & $1000$ & $57$ & gender & credit decision\\
         {Compas}& $5278$ & $10$ & ethnicity & recidivism decision\\
        \bottomrule
    \end{tabular}
    \caption{Dataset statistics.}
    \label{tab:datastats}
\end{table}

\begin{table*}[t!]
    \centering
    
    \renewcommand{\arraystretch}{1.5}
    \begin{tabular}{l|c|ccccc||c|ccccc}
    \toprule
    
       Dataset & LR & $\Delta_{\text{SP}}$ & $\Delta_{\text{EO}}$ &  Our $\delta_{\hat y}$ & $\delta_{\hat y}$ in~\cite{kose2024fairgat} &Accuracy & DNN & $\Delta_{\text{SP}}$ & $\Delta_{\text{EO}}$ &  Our $\delta_{\hat y}$ & $\delta_{\hat y}$ in~\cite{kose2024fairgat} &Accuracy\\\midrule
        \multirow{3}{*}{Adult} &$\bm{w}_1$ & 0.192 & 0.092 &  10.32 & 37.87 & 83.16\%
        &$\mat{W}_1$ & 0.138 & 0.051 & 98.37 & 10694.96 & 80.48\%\\
        
        &$\bm{w}_2$ & 0.074 & 0.018  & 8.87 & 37.83 & 82.24\%&
        $\mat{W}_2$ & 0.095 & 0.021 & 56.99 & 6437.04 & 78.83\%\\
        
        &$\bm{w}_3$ & 0.048 & 0.002 & 1.17 & 37.46 & 78.80\% & $\mat{W}_3$ & 0.007 & 0.005 & 54.79 & 6094.85 & 79.69\%\\ \hline\hline
        
        \multirow{3}{*}{German} &$\bm{w}_1$ & 0.102 & 0.110 & 11.74 & 44.66 & 76.33\% 
        & $\mat{W}_1$ & 0.087 & 0.063 & 2021.59 & 113215.41 & 75.00\%\\
        
        &$\bm{w}_2$ & 0.066 & 0.050 & 10.56 & 44.66 & 71.67\%
        &$\mat{W}_2$ &  0.049 & 0.030 & 1590.23 & 88973.34 & 70.67\%\\

        &$\bm{w}_3$ & 0.021 & 0.017 & 8.80 & 44.65 & 75.00\% &
        $\mat{W}_3$ & 0.014 & 0.000 & 301.22 & 19380.94 & 71.00\%\\ \hline\hline
        
        \multirow{3}{*}{Compas} &$\bm{w}_1$ & 0.232 & 0.235 & 8.08 & 10.49 & 66.02\% &$\mat{W}_1$ & 0.249 & 0.250 & 172.99 & 16025.69 & 66.59\%\\
        
        &$\bm{w}_2$ & 0.153 & 0.150 & 2.25 & 10.03 & 61.94\%
        &$\mat{W}_2$ & 0.161 & 0.163 & 65.43 & 2847.70 & 61.71\%\\
        
        &$\bm{w}_3$ & 0.019 & 0.026 & 1.18 & 10.03 & 53.44\%&$\mat{W}_3$ & 0.000 & 0.000 & 9.95 & 731.43 & 52.24\%\\ \bottomrule
        
    \end{tabular}
    \caption{Fairness measures and bounds of logistic regression and DNN models.}
    \label{tab:LRfair}
\end{table*}

\section{Experiments}\label{sec:exp}

We have fully implemented our system, \name, for proving ML fairness in zero-knowledge. We evaluate its performance in this section. 

\paragraph{Software and hardware.} \name is implemented in Rust. We use the Goldilock field $p=2^{64}-2^{32}+1$ as the prime of the base field, and use its 2nd extension field $\F_{p^2}$ in our protocol to achieve 100+ bits of security. We implement the field arithmetic based on the Arkworks library~\cite{arkworks} and use the sumcheck protocol from the library. \yupeng{Update} We implement the polynomial commitment scheme in Brakedown~\cite{golovnev2023brakedown} with reference to the existing library of~\cite{plonkish}. The ML models are trained and tested in PyTorch~\cite{paszke2019pytorch}. The implementation consists of 12,400 lines of Rust code. 

All of the experiments were executed on a Macbook Air equipped with an M3 CPU with 4.05 GHz and 16GB of RAM. Our implementation is serialized and only utilizes a single CPU core.   

\paragraph{Datasets.} The experiments were conducted on three datasets that are commonly used in the ML fairness literature: Adult~\cite{becker2024adult}, German~\cite{dua2017uci} and Compas~\cite{angwin2022machine}. 
\begin{itemize}
    \item \textbf{Adult dataset~\cite{becker2024adult}:} There are $45,222$ data samples with $14$ features in the dataset, where one-hot encodings are used for the categorical features leading to $F=38$. We use gender as the sensitive attribute, where the classification labels are created based on the income levels. For training, $75\%$ of the samples are used, where the models are evaluated on the remaining data samples. 
    \item \textbf{German dataset~\cite{dua2017uci}:} This dataset contains $1000$ samples, each with $20$ features. For categorical features, we use one-hot encodings resulting in a feature dimension $F=57$ in the model training. For this data, gender is used as the sensitive attributes, and the credit decision for a person constitutes the corresponding labels. For model training, we use $70\%$ of the available data samples, while the remaining data are used to evaluate the model.
    \item \textbf{Compas dataset~\cite{angwin2022machine}:} This dataset consists of $5278$ samples with feature number $10$. Here, we use ethnicity as the sensitive attribute, while the goal is to predict the recidivism status. The models are trained over $3536$ samples, and the remaining $1742$ samples are used to test our models. 
\end{itemize}

\subsection{Fairness Bounds on ML Models}\label{subsec:boundexp}

We first evaluate the effectiveness of our new bounds on ML fairness. 

\paragraph{Logistic Regression.} We train three different LR models for each dataset with different fairness-utility trade-offs. An LR model $\bm{w}$ is a vector of the same dimension as a data sample. To get fair models, inspired by the terms in the fairness bounds,  we apply spectral normalization~\cite{kose2024fairgat} and a fairness-aware feature masking strategy similar to the one employed in~\cite{kose2022fairaug}. Specifically, the proposed fair feature masking first finds the features having the largest disparity for different sensitive groups (based on $\boldsymbol{\delta}_{h}^{l-1}$ in Lemma \ref{lemma:linear2}). Then, it masks out a certain portion of the input features. Note that masking is applied to the input features at every DNN layer differing from~\cite{kose2022fairaug} to mitigate bias more effectively. 
Note that the training process is independent of the zkSNARK scheme. Our system proves the fairness score of a model regardless of how it is trained. 


Table~\ref{tab:LRfair} (left side) shows the fairness, accuracy and bounds of multiple LR models for different datasets. The bound in~\cite{kose2024fairgat} is adapted to logistic regression by replacing the spectral norm of a matrix with the L2 norm of the model as a vector. $\Delta_x$ and $\bm{\delta}_x$ are computed directly from the corresponding datasets. As shown in the table, our new bound is always tighter than the bound in~\cite{kose2024fairgat}. Moreover, it better reflects the fairness of the LR model. For example, for the adult dataset, $\bm{w}_1$ does not have any fairness mechanism, $\bm{w}_2$ is trained with spectral normalization, and $\bm{w}_3$ is trained with both spectral normalization and feature masking. Therefore, the fairness gets better while the accuracy drops. This is a well-known trade-off in machine learning fairness~\cite{pessach2022review}. The bound in~\cite{kose2024fairgat} does not change by much across the three models. This is because the bound is dominated by the term that is the square-root of the dimension $F$, which is the same in the three models. By contrast, our new bound strictly decreases with $\Delta_\text{SP}$ and $\Delta_\text{EO}$. Similar relationship can be observed from the other two datasets as well. Therefore, our fairness score better reflect the fairness of the LR models.

\paragraph{DNN.} We also train three different DNN models for each dataset with different fairness levels. For the Adult dataset, a three layer multi-layer perceptron (MLP) is used with hidden dimension $128$. For German and Compas, we employ two layer MLPs with hidden dimensions $128$ and $64$, respectively. For all MLPs, the activation function is selected as sigmoid. We apply the same technique of spectral normalization and feature masking to obtain fair models.




The metrics of the models are reported in Table~\ref{tab:LRfair} (right side). Our bound is much tighter than the bound in~\cite{kose2024fairgat} for DNN. The gap is more significant than that of the logistic regression models, ranging from $43.5\times$ to $112.9\times$ for different datasets and models. This is because of the term $\sqrt{F^\ell}$ and the infinity norm of the matrix $\mat{W}^\ell$ to bound $\Delta^\ell_z$ of intermediate values in~\cite[Equation 5]{kose2024fairgat}, which accumulate with the number of layers in DNN. Moreover, our bound decreases with $\Delta_\text{SP}$ and $\Delta_\text{EO}$ as well. 

\new{\paragraph{Meaning of fairness scores in practice.} Although tighter than the bounds in~\cite{kose2024fairgat}, the values of our bounds are larger than 1 as shown in Table~\ref{tab:LRfair}. The loose upper bound is a pervasive problem in ML fairness due to the layer-by-layer bounding. Deriving tight bounds is an open problem in the area of ML fairness. 

Therefore, in our experiments above, we empirically demonstrate that the fairness score computed by the new bounds is a good indication of the fairness for the same model architecture and aggregated statistics. For different models, the score is affected by the number of layers and the dimensions. To use the fairness score in practice, we envision two approaches: (1) the score can be normalized by the number of layers and dimensions to provide a unified scale for different architectures; (2) “reference scores” for each architecture can be computed using public models. }

\subsection{Performance of \name.}\label{subsec:zkexp}

In this section, we benchmark the performance of our \name system. 

\new{\paragraph{Baselines and their implementations.} We also compare \name with two baselines. The first one is the naive approach of running multiple ZK inferences and calculating the disparities. The second baseline is the recent work of~\cite{franzese2024oath} named OATH. 

To ensure fair comparisons, we implement the baselines ourselves using the same sumcheck, GKR and polynomial commitment protocols over the same field as \name and test them on the same machine. We also implement known optimizations for these baselines, including the optimized sumcheck protocol for matrix multiplications, and the lookup arguments for the activation functions of sigmoid and ReLU. The performance of the first baseline is obtained by running ZK inferences for all data samples in the dataset. 
}

In the second baseline in~\cite{franzese2024oath}, the paper proposed a probabilistic approach to test fairness using ZK inferences to reduce the overhead. Note that OATH works in a different setting where there is an auditor staying online during the inferences with clients to decide when to ask for a ZKP of inference. We do not require such a third party. Nevertheless, we assume that there is a single ZKP for all inferences tested in OATH, which is a lower bound of its running time. We set the number of inferences as 7,600 as used in~\cite{franzese2024oath}, which achieves a good trade-off between runtime and robustness. As the German dataset and Compas dataset do not have that many data samples, we duplicate the samples to reach 7,600. The ZKP is again implemented using the same optimized zkSNARK back-end and the results are obtained on the same machine. We do not compare to the performance of~\cite{shamsabadi2022confidential} and~\cite{yadav2024fairproof}. The former is proving the fair training of decision trees, and the latter is proving the individual fairness, and thus they are not directly comparable to \name. 

\begin{table}[t!]
    \centering
    \renewcommand{\arraystretch}{1.5}
    \begin{tabular}{l|c||c|c|c}
    \toprule
        Dataset & Scheme & Prover time & Proof size & Verifier time\\\midrule
        \multirow{3}{*}{Adult}& Naive & 5368ms & 168MB & 294ms\\
        & OATH~\cite{franzese2024oath} & 962ms & 123MB & 124ms \\
        & Ours & 3ms & 1.6MB & 1ms\\\hline\hline
         \multirow{3}{*}{German}& Naive & 365ms & 87MB & 75ms\\
        &OATH~\cite{franzese2024oath} & 964ms &123MB & 124ms\\
        & Ours & 6ms & 1.6MB & 2ms\\\hline\hline
         \multirow{3}{*}{Compas}& Naive & 478ms &  115MB & 93ms\\
        & OATH~\cite{franzese2024oath} & 477ms &115MB & 92ms\\
        & Ours & 5ms & 1.5MB & 2ms\\\bottomrule
    \end{tabular}
    \caption{Performance of \name for logistic regression.}
    \label{tab:zkpLR}
\end{table}

As the performance of zkSNARKs is mainly determined by the size of the instance, i.e., the dimensions of the ML models, but not their values and other metrics, we choose one LR model and one DNN model for each dataset to report the performance. 

\paragraph{Logistic regression.} The performance for logistic regression is shown in Table~\ref{tab:zkpLR}. For the Adult dataset, it merely takes 3 milliseconds to generate the proof in \name. This is because the computation only involves two inner products of size $F$ and an element-wise absolute value on the model, as shown in Equation~\ref{eq:LRfair}. This is 1789$\times$ faster than the naive approach of running ZK inferences on all data samples, and 321$\times$ faster than OATH. 

The prover time of \name does not change much for the LR models on the German and Compas dataset. This is because although the number of features are different, our prover time is completely dominated by the lookup arguments to compute the absolute values, which is linear in the size of the lookup table. As we are using the same table to validate the range of numbers and compute absolute values, the prover time is almost the same. For OATH, as explained earlier, we duplicate the samples in German and Compas to reach 7,600, thus the prover time of OATH is slower than the naive approach for German, and is almost the same for Compas, as the dimensions are padded to a power of 2 to run the sumcheck protocol. \name is 61--96$\times$ faster than the naive approach, and 96--161$\times$ faster than OATH in these cases. 

\yupeng{update} The proof size of all schemes are quite large, ranging from several MBs to more than 100 MBs. This is solely due to the choice of our zkSNARK back-end, particularly the polynomial commitment scheme in Brakedown~\cite{golovnev2023brakedown}. The proof size is asymptotically square-root in the size of the witness, and is known to be concretely large due to the low distance of the underlying error-correcting code. The proof size can be easily improved by switching to the polynomial commitment schemes in~\cite{spartan,libra,virgo,xie2022orion}, which are all compatible with multivariate polynmomials and sumcheck protocols. We primarily focus on achieving the best prover efficiency in our implementation, and thus choose the scheme in~\cite{golovnev2023brakedown}. Nevertheless, the comparison between different approaches in Table~\ref{tab:zkpLR} still holds, as they are all implemented using the same back-end. As shown in the table, both the proof size and the verifier time of \name are significantly faster than those of the baselines.

\begin{table}[t!]
    \centering
    \renewcommand{\arraystretch}{1.5}
    \begin{tabular}{l|c||c|c|c}
    \toprule
        Dataset & Scheme & Prover time & Proof size & Verifier time\\\midrule
        
        \multirow{3}{*}{Adult}& Naive & 40.28s & 614MB & 1.36s \\
        & OATH~\cite{franzese2024oath} & 5.82s & 451MB & 0.60s\\
        & Ours & 0.47s & 258MB & 0.15s\\\hline\hline
        
         \multirow{3}{*}{German}& Naive & 0.87s & 212MB & 0.26s \\
        & OATH~\cite{franzese2024oath} & 3.61s& 281MB & 0.36s\\
        & Ours & 0.28s & 174MB & 0.10s \\\hline\hline
        
         \multirow{3}{*}{Compas}& Naive & 1.73s & 231MB & 0.23s \\
        & OATH~\cite{franzese2024oath} & 1.69s& 231MB& 0.23s\\
        & Ours & 0.16s & 86MB & 0.05s\\\bottomrule
    \end{tabular}
    \caption{Performance of \name for DNN models.}
    \label{tab:zkpDNN}
\end{table}

\paragraph{DNN.} The performance and comparison for DNN models are shown in Table~\ref{tab:zkpDNN}. Similar to LR models, the prover time of \name is significantly faster than the naive approach and OATH. On the Adult dataset, it only takes 0.47 second to generate a proof in \name, which is $86\times$ faster than the naive approach, and $12\times$ faster than OATH. The dominating cost of \name for DNN is proving the spectral norm of the weight matrices. Note that the improvement over the naive approach on the German dataset is only $3.1\times$. This is because the number of data samples is only 1000. Generating a ZKP for 1000 inferences is only $3.1\times$ slower than proving the spectral norm for this simple DNN with 2 layers. The improvement is much larger on larger datasets and more complicated ML models.  

\begin{table}[t!]
    \centering
    \renewcommand{\arraystretch}{1.5}
    \begin{tabular}{l|c||c|c|c}
    \toprule
        Dataset & Total size & Prover time & Proof size & Verifier time\\\midrule
        {Adult} & $1.7\times 10^6$ & 13.6s & 314MB& 0.56s\\
         {German} & $5.7\times 10^4$ & 0.50s& 134MB& 0.09s\\
         {Compas}& $5.2\times 10^4$ & 0.80s& 173MB&0.12s\\
        \bottomrule
    \end{tabular}
    \caption{Performance of \name for computing meta disparity from data sets. Total size is the number of samples times the number of features $N\times F$.}
    \label{tab:zkpdelta}
\end{table}

\paragraph{Aggregated Statistics.} We also report the prover time to compute $\Delta_x,\bm{\delta}_x$ from a dataset in \name, which is used in the second setting we consider. As shown in Table~\ref{tab:zkpdelta}, it takes 13.6 seconds to generate a proof for the Adult dataset, 0.5s for German and 0.8s for Compas. The performance only depends on the dataset, but not on the ML models. When combined with the zkSNARK proving the fairness scores of ML models as in the second setting we consider, the total prover time is still faster than the naive approach and OATH for the DNN models above. It is slower than the baselines for LR models as the models are too small and computing the inferences is faster than computing the aggregated statistics. However, as we will show in the next subsection, \name has significant advantages when scaling to large DNN models.

\paragraph{Micro-benchmark.} To further demonstrate the improvement of the zkSNARK protocols proposed in this paper, we compare the performance of FairZK with a general-purpose ZKP scheme proving the fairness bound. As the dominating cost is the spectral norm, we provide a micro-benchmark on this gadget. We use the general-purpose zkSNARK library of Gnark~\cite{gnark} to implement the eigenvalue decomposition algorithm~\cite{stoll2013linear} to compute the spectral norm.

We vary the dimension of the matrix from $2^7\times 2^7$ to $2^{12}\times 2^{12}$ and Figure~\ref{fig:micro} shows the prover time. As we were not able to run the smallest instance in Gnark, its prover time is estimated by extrapolation using the number of R1CS constraints obtained from smaller instances. As shown in the figure, our scheme is significantly faster than Gnark. The prover time is $5.7\times$ faster for a matrix of size $2^{7}\times 2^{7}$, and is $1045\times$ faster for a matrix of size $2^{12}\times 2^{12}$. The speedup is almost entirely due to our new zkSNARK verifying the spectral norm in Figure~\ref{fig:spectral}, which reduces the complexity from $O(F^3)$ in the naive approach to $O(F^2)$, linear in the size of the matrix. 

\begin{figure}[t!]
    \centering
%
%
\definecolor{colorA}{HTML}{0000FF}
\definecolor{colorC}{HTML}{e6194B}
\definecolor{colorB}{HTML}{000000}
\definecolor{colorD}{HTML}{00FF00}
\definecolor{colorE}{HTML}{E9D66B}
\definecolor{colorF}{HTML}{FF9966}
\definecolor{colorH}{HTML}{FB607F}
\definecolor{colorG}{HTML}{966FD6}

\begin{tikzpicture}
\begin{axis}[%
width=.85\columnwidth,
height=.5\columnwidth,
ymode=log,
scale only axis,
xmin=0.8,
xmax=6.2,
xtick={1, 2, 3, 4, 5, 6},
xlabel={Size of Matrix},
xticklabels={$2^7\times 2^7$, $2^8\times 2^8$, $ 2^9\times 2^9$, $2^{10}\times 2^{10}$, $ 2^{11}\times 2^{11} $, $2^{12}\times 2^{12}$},
max space between ticks=20,
ytick={0.1,1,10,100,1000,10000,100000},
ymax = 500000,
ymin= 0.1,
ymajorgrids,
ylabel={Prover Time (seconds)},
y label style={yshift=-0.45cm},
axis background/.style={fill=white},
legend style={
    at={(0.02, 0.98)},
    anchor=north west,
    legend cell align=left,
    },
label style={font=\small},
tick label style={font=\small},
ytick style={draw=none},
xmajorgrids,
grid style={line width=.5pt, draw=gray!90,dashed},
major grid style={line width=.2pt,draw=gray!50},
%
%
typeset ticklabels with strut,
]		

\addplot [color=colorA, mark=square*, mark options={scale=0.7,solid, colorA}]
  table[row sep=crcr,y expr=\thisrow{Y}]{%
X   Y
0 0\\
1 	0.3\\
2	0.615\\
3  1.166\\
4	2.980\\
5 9.289\\
6 37.937\\
};

\addplot [color=colorC, dotted, thick, mark=*, mark options={scale=0.7,solid, colorC}]
  table[row sep=crcr,y expr=\thisrow{Y}]{%
X   Y
0 0\\
1 	 1.696   \\
2	 11.557   \\
3    84.498   \\
4	 644.359   \\
5    5028.776   \\
6    39726.647   \\
};

\addlegendentry{Our scheme \name}
\addlegendentry{Gnark}

\end{axis}
\end{tikzpicture}%
    \caption{Prover time of spectral norm in \name and a generic zkSNARK. }
    \label{fig:micro}
\end{figure}
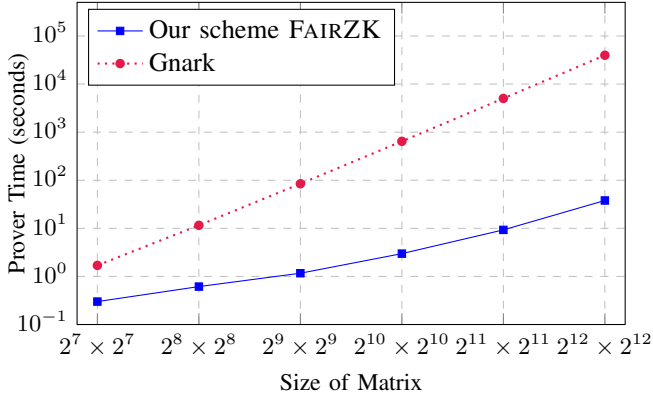

\begin{table}[t!]
    \centering
    \renewcommand{\arraystretch}{1.5}
    \begin{tabular}{c|c|c|c}
    \toprule
        Param  & Prover time & Proof size & Verifier time\\\midrule
        {808K} & 3.48s & 412MB& 0.49s \\\hline
       {25M} &  51.32s &  648MB & 2.97s\\\hline
        {47M} & 342.74s& 1175MB& 15.77s\\
        \bottomrule
    \end{tabular}
    \caption{Performance of \name on large DNN models.}
    \label{tab:scalability}
\end{table}

\subsection{Scalability of \name}\label{subsec:expscale}

Finally, we test the scalability of \name on large DNN models. We train 3 DNN models on the Adult dataset. The first model has 5 layers and 808 thousand parameters. The second model has 8 layers and 25.2 million parameters. The third model has 7 layers and 47.3 million parameters. These models are unnecessarily big for the Adult dataset with 37 features and 45,000 samples, but the main purpose here is to demonstrate the scalability of our zkSNARK system. 

As shown in Table~\ref{tab:scalability}, \name can easily scale to large models. It only takes 343 seconds to generate the proof of fairness score for a DNN model with 47 million parameters. In~\cite[Table 3]{franzese2024oath}, the running time for a model with 42.5 million parameters is estimated to be 115 days, and the naive approach of inferences would take 6227 days. \name is the first to scale to such a big model in practice, and the prover time is 4 orders of magnitude faster than OATH~\cite{franzese2024oath}.

        


\section{Conclusion}

In this paper, we develop a scalable system for proving the fairness of ML models in zero-knowledge.  It improves the prover time by several orders of magnitude compared to prior work and scales to large models with 47 million parameters for the first time. We believe our approach, in particular the new fairness bounds, can be extended to other ML models such as graph neural networks and graph attention networks to prove their fairness, which is left as a future work. In addition, our approach serves as a framework of building efficient ZKPs for ML fairness and urges the development of better fairness bounds in the future that only depends on the model parameters and aggregated information of input.

\section*{Acknowledgments}

This material is in part based upon work supported by the National Science Foundation (NSF) under Grant No. 2401481, 2412484 and a Google Research Scholar Award. Any opinions, findings, and conclusions or recommendations expressed in this material are those of the author(s) and do not necessarily reflect the views of these institutes.


\bibliographystyle{abbrv} 
\bibliography{zkp,fair}

\appendices

\section{Additional Preliminaries}\label{app:sumcheck}

\begin{figure}[t!]
{\centering
\framebox{\parbox{.99\linewidth}{
\begin{protocol}[\textbf{Sumcheck protocol}]
	\label{prot::sumcheck}
        The protocol has $k$ rounds. 
	\begin{itemize}
		\item In the first round, $\P$ sends $\V$ a univariate polynomial $$f_1(x_1)\stackrel{\text{def}}{=}\sum\limits_{b_2,\ldots,b_k\in\{0,1\}}f(x_1,b_2,\ldots,b_k)\, ,$$ $\V$ checks $S=f_1(0)+f_1(1)$. Then $\V$ sends $\P$ a random challenge $r_1\in\F$.
		\item In the $i$th round, $\P$ sends $\V$ a univariate polynomial
		$$f_{i}(x_{i})\stackrel{\text{def}}{=}\sum\limits_{b_{i+1},\ldots,b_k\in\{0,1\}}f(r_1,\ldots, r_{i-1}, x_{i}, b_{i+1},\ldots, b_{k})\, ,$$ 
		$\V$ checks $f_{i-1}(r_{i-1})=f_{i}(0)+f_{i}(1)$, and sends $\P$ a random challenge $r_{i}\in\F$.
		\item In the last round, $\P$ sends a univariate polynomial $$f_{k}(x_{k})\overset{def}{=}f(r_1, r_2, \ldots, r_{l-1}, x_{k})\, ,$$ $\V$ checks $f_{\ell-1}(r_{k-1})=f_{k}(0)+f_{k}(1)$. 
            \item Finally, $\V$ generates a random challenge $r_{k}\in\F$ and checks if $f_{k}(r_k) = f(r_1, r_2, \ldots, r_k)$. $\V$ outputs 1 if and only if all the checks pass.  
	\end{itemize}
\end{protocol}}}}
\end{figure}

The sumcheck protocol is presented in Protocol~\ref{prot::sumcheck}. The GKR protocol is an interactive proof for layered arithmetic circuits proposed in~\cite{GKR}. For a circuit with $d$ layers, let $V_i:\binary^{c_i}\rightarrow\F$ denote the function defined by the output of the gates in the $i$-th layer of the circuit, where $C_i$ is the number of gates in the $i$-th layer and $c_i = \log C_i$, assuming $C_i$ is a power of 2 w.l.o.g. $V_0$ is defined by the output of the circuit, while $V_d$ is defined by the input of the circuit. Let $\tV_i$ be the MLE of $V_i$, then $\tV_{i-1}$ can be written as a sumcheck equation of $\tV_i$ as for all $\bm{x}\in\binary^{c_{i-1}}$:
\begin{align}
    \tV_{i-1}(\bm{x}) = \sum_{\bm{y},\bm{z}\in\binary^{c_i}}&(\tadd_i(\bm{x},\bm{y},\bm{z})(\tV_i(\bm{y})+\tV_i(\bm{z}))\nonumber\\
    &+\tmult_i(\bm{x},\bm{y},\bm{z})\tV_i(\bm{y})\tV_i(\bm{z})) \, ,
\end{align}
where $\tadd$ and $\tmult$ are MLEs of wiring predicates such that $add_i(\bm{x},\bm{y},\bm{z})=1$ ($mult_i(\bm{x},\bm{y},\bm{z})=1$) if and only if $(\bm{x},\bm{y},\bm{z})$ connects to an addition (multiplication) gate in layer $i$. With the equation and the sumcheck protocol, the GKR protocol is presented in Protocol~\ref{prot:gkr}.  

\begin{figure}[t!]
{\centering
\framebox{\parbox{.99\linewidth}{
\new{

\begin{protocol}[\textbf{GKR protocol}]
	\label{prot:gkr}
        The protocol proceeds layer by layer.  
	\begin{itemize}
		\item For the output layer, $\P$ sends $\V$ the claimed output of the circuit, from which $\V$ defines its MLE $\tV_0$ and evaluates it at a random point $\tV_0(\bm{r}_x^{(0)})$.

        \item For $i=1,2,\ldots, d$, 
        $\V$ sends $\bm{r}_x^{(i-1)}$ to $\P$, and executes the sumcheck protocol with $\P$ on:
        \begin{align*}
            \tV_{i-1}(&\bm{r}_x^{(i-1)}) =&\\
            &\sum_{\bm{y},\bm{z}\in\binary^{c_{i}}}&(\tadd_{i}(\bm{r}_x^{(i-1)},\bm{y},\bm{z})(\tV_{i}(\bm{y})+\tV_{i}(\bm{z}))\\
            &&+\tmult_{i}(\bm{r}_x^{(i-1)},\bm{y},\bm{z})\tV_{i}(\bm{y})\tV_{i}(\bm{z})).
        \end{align*}
        
        \item At the end of the sumcheck protocol, $\V$ needs to evaluate the polynomial $f$ in the sumcheck at $f(\bm{r}^{(i-1)}_x, \bm{r}^{(i)}_y,\bm{r}^{(i)}_z)$. To do so, $\V$ evaluates $\tadd_i(\bm{r}^{(i-1)}_x, \bm{r}^{(i)}_y,\bm{r}^{(i)}_z)$ and $\tmult_i(\bm{r}^{(i-1)}_x, \bm{r}^{(i)}_y,\bm{r}^{(i)}_z)$ locally, and asks the prover to send $\tV_i(\bm{r}^{(i)}_y),\tV_i(\bm{r}^{(i)}_z)$, and checks that they are consistent with the last step of the sumcheck protocol.

        \item $\V$ and $\P$ invokes a protocol to combine $\tV_i(\bm{r}^{(i)}_y),\tV_i(\bm{r}^{(i)}_z)$ into a single evaluation $\tV_i(\bm{r}^{(i)}_x)$, which is used for the sumcheck protocol in layer $i+1$. See~\cite{xielibra} for the detailed description.

		\item For the input layer, $\V$ checks that $\tV_d(\bm{r}^{(d)}_x)$ is correctly computed.  
	\end{itemize}
\end{protocol}}}}
}
\end{figure}

\paragraph{Polynomial commitment.} A PCS consists of the following algorithms:
\begin{itemize}
    \item $\pp\leftarrow\pcs.\Setup(1^\lambda,\mathcal{F})$: Given the security parameter and a family of polynomials, output the public parameters $\pp$.
    \item $\com_f\leftarrow\pcs.\Commit(f,\pp)$: On input a polynomial $f$ and the public parameter pp, output the commitment. 
    \item $(v,\pi)\leftarrow\pcs.\Open(f,u,\pp)$: Given an evaluation point $u$, output the evaluation $v$ and the proof $\pi$.
    \item $\binary\leftarrow\pcs.\Verify(\com_f,u,\pi,\pp)$: Verify the proof against the claim $f(u) = v$. 
\end{itemize}
A PCS scheme also satisfies correctness, knowledge soundness and zero-knowledge. Formal definitions can be found in~\cite{hyrax,zkvpd}. 

\paragraph{zkSNARK from GKR and PCS.} An argument scheme can be constructed by combining the GKR protocol with a PCS. The prover commits to the witness $w$ by $\com_w \leftarrow\pcs.\Commit(\tilde{w},\pp)$, where $\tilde{w}$ is the MLE defined by $w$. The prover and the verifier then execute the GKR protocol on a layered arithmetic circuit defining the relationship $\mathcal{R}$. At the end of the GKR protocol, the verifier needs to check $\tV_i(\bm{r}^{(d)}_x)$, the evaluation of the MLE of the input, as shown in Protocol~\ref{prot:gkr}. The prover opens the evaluation by $\pcs.\Open(\tilde{w},\bm{r}^{(d)}_x, \pp)$. The verifier checks the proof with $\pcs.\Verify$, and checks that the evaluation is consistent with the last step of the GKR protocol. 

This scheme can be made zero-knowledge using the zero-knowledge sumcheck proposed in~\cite{zksumcheck,xielibra}. It can be made non-interactive using the Fiat-Shamir transformation~\cite{fiat1986prove}. The scheme can support general circuits with arbitrary connections without any overhead on the prover time using the techniques in~\cite{zhang2021doubly}. 


\paragraph{Lookup arguments.} As shown in~\cite{habock2022multivariate}, to prove that $A = \ms{Lookup}(T)$, it suffices to show that 
\[
\sum_{\bm{x}\in\binary^{\log n}}\frac{1}{\gamma+\tilde{A}(\bm{x})} = \sum_{\bm{x}\in\binary^{\log N}}\frac{\tilde{m}(\bm{x})}{\gamma+\tilde{T}(\bm{x})} 
\]
at a random value $\gamma$, where $\tilde{A},\tilde{T}$ are MLEs defined by $A,T$ and $\tilde{m}$ is the MLE of $m$, the cardinality of elements in $T$ in $A$. This relationship can be proven by sumcheck protocols and polynomial commitments. Define MLEs $\tilde{h}_0(\bm{x}) = \frac{1}{\gamma+\tilde{A}(\bm{x})}$ for all $\bm{x}\in\binary^{\log n}$ and $\tilde{h}_1(\bm{x}) = \frac{\tilde{m}(\bm{x})}{\gamma+\tilde{T}(\bm{x})}$ for all $\bm{x}\in\binary^{\log N}$. The prover commits to $\tilde{A},\tilde{T},\tilde{m},\tilde{h}_0,\tilde{h}_1$ using $\pcs.\Commit$, and executes the sumcheck protocol on:
\begin{enumerate}
    \item $\sum_{\bm{y}\in\binary^{\log n}} \tilde{eq}(\bm{r}_x,\bm{y}) (\tilde{h}_0(\bm{y})(\gamma+\tilde{A}(\bm{y}))-1) = 0$.
    \item $\sum_{\bm{y}\in\binary^{\log N}} \tilde{eq}(\bm{r}'_x,\bm{y}) (\tilde{h}_1(\bm{y})(\gamma+\tilde{T}(\bm{y}))-\tilde{m}(\bm{y}) = 0$.
    \item $\sum_{\bm{x}\in\binary^{\log n}} \tilde{h}_0(\bm{x}) - \sum_{\bm{x}\in\binary^{\log N}}\tilde{h}_1(\bm{x}) = 0$.
\end{enumerate}
Finally, the prover opens these polynomial commitments at the random points checked in the last step of the sumcheck protocols. The verifier verifies the sumcheck protocol as well as the proofs of the PCS. 

This protocol is directly compatible with the zkSNARK based on GKR and PCS. The elements in $A$ can be from one layer of the circuit, and the first equation above can be executed together with the sumcheck in the GKR protocol for that layer. The evaluation of $\tilde{A}$ is reduced to previous layers in GKR, while the evaluations of $\tilde{T},\tilde{m},\tilde{h}_0,\tilde{h}_1$ are opened and verified using PCS.

\section{Deferred Proofs}\label{app:proof}

\subsection{Proof of Lemma~\ref{lemma:nonlinear}}

By the definition of $\delta_{h,i}^\ell$ and $\bm{h}^\ell$,
\begin{equation}\label{eq:defs2}
\begin{split}
    |\delta_{h,i}^\ell|&:=\left|\operatorname{mean}(h_{j,i}^\ell|s_j=0) - \operatorname{mean}(h_{j,i}^\ell|s_j=1)\right|\\
    &= \left|\frac{1}{|\mathcal{S}_{0}|} \sum_{v_j \in \mathcal{S}_{0}} h^\ell_{j,i}-\frac{1}{|\mathcal{S}_{1}|} \sum_{v_j \in \mathcal{S}_{1}} h^\ell_{j,i}\right|\\
    &= \left|\frac{1}{|\mathcal{S}_{0}|} \sum_{v_j \in \mathcal{S}_{0}} \sigma(z^\ell_{j,i})-\frac{1}{|\mathcal{S}_{1}|} \sum_{v_j \in \mathcal{S}_{1}} \sigma(z^\ell_{j,i})\right|\\
\end{split}
\end{equation}
For $i=0,1,\ldots, F_\ell-1$. For simplicity, we drop the superscript $\ell$ in the proof. We can write $z_{j,i}= \bar{z}_i^{(s)} + \delta^{(s)}_{j,i}$, $\forall v_j \in \mathcal{S}_{s}$ , where $\bar{z}_i^{(s)} = \frac{1}{|\mathcal{S}_{s}|} \sum_{v_j \in \mathcal{S}_{s}} z_{j,i}$ for $s=0,1$. By Equation~\ref{eq:lipschitz},
\begin{equation}
\label{eq:main_ineq}
\begin{split}
\operatorname{\sigma}(\bar{z}_i^{(0)}) - L|\delta^{(0)}_{j,i}|  &\leq \operatorname{\sigma}(z_{j,i})= \operatorname{\sigma}((\bar{z}_i^{(0)}) + \delta^{(0)}_{j,i})\\
&\leq \operatorname{\sigma}((\bar{z}_i^{(0)})) + L|\delta^{(0)}_{j,i}|, \forall v_j \in \mathcal{S}_{0} \\
  \end{split}
\end{equation}
and 
\begin{equation}
\label{eq:main_ineq2}
\begin{split}
\operatorname{\sigma}(\bar{z}_i^{(1)}) - L|\delta^{(1)}_{j,i}|  &\leq \operatorname{\sigma}(z_{j,i})= \operatorname{\sigma}((\bar{z}_i^{(1)}) + \delta^{(1)}_{j,i})\\
&\leq \operatorname{\sigma}((\bar{z}_i^{(1)})) + L|\delta^{(1)}_{j,i}|, \forall v_j \in \mathcal{S}_{1} \\
  \end{split}
\end{equation}

Based on Equations \eqref{eq:main_ineq}, and \eqref{eq:main_ineq2}, we have:
\begin{equation}
\begin{split}
&\frac{1}{|\mathcal{S}_{0}|} \sum_{v_j \in \mathcal{S}_{0}} \left(\operatorname{\sigma}(\bar{z}_i^{(0)}) - L|\delta^{(0)}_{j,i}|\right ) - \\
&\qquad\frac{1}{|\mathcal{S}_{1}|} \sum_{v_j \in \mathcal{S}_{1}}  \left(\operatorname{\sigma}(\bar{z}_i^{(1)}) + L|\delta^{(1)}_{j,i}| \right)\\
\leq &\frac{1}{|\mathcal{S}_{0}|} \sum_{v_j \in \mathcal{S}_{0}} \sigma(z_{j,i})-\frac{1}{|\mathcal{S}_{1}|} \sum_{v_j \in \mathcal{S}_{1}} \sigma(z_{j,i})\\
\leq &\frac{1}{|\mathcal{S}_{0}|} \sum_{v_j \in \mathcal{S}_{0}} \left( \operatorname{\sigma}(\bar{z}_i^{(0)}) + L|\delta^{(0)}_{j,i}|\right) - \\
&\qquad\frac{1}{|\mathcal{S}_{1}|} \sum_{v_j\in \mathcal{S}_{1}}  \left(\operatorname{\sigma}(\bar{z}_i^{(1)}) + L|\delta^{(1)}_{j,i}|\right)
\end{split}
\end{equation}
Therefore, 
\begin{equation}
\label{eq:before_norm}
\begin{split}
&\operatorname{\sigma}(\bar{z}_i^{(0)}) - \operatorname{\sigma}(\bar{z}_i^{(1)}) - \\
&\qquad L(\frac{1}{|\mathcal{S}_{0}|} \sum_{v_j \in \mathcal{S}_{0}} |\delta^{(0)}_{j,i}| - \frac{1}{|\mathcal{S}_{1}|} \sum_{v_j \in \mathcal{S}_{1}}  |\delta^{(1)}_{j,i}|)\\
\leq &\frac{1}{|\mathcal{S}_{0}|} \sum_{v_j \in \mathcal{S}_{0}} \sigma(z_{j,i})-\frac{1}{|\mathcal{S}_{1}|} \sum_{v_j \in \mathcal{S}_{1}} \sigma(z_{j,i})\\
\leq &\operatorname{\sigma}(\bar{z}_i^{(0)}) - \operatorname{\sigma}(\bar{z}_i^{(1)}) + \\
&\qquad L(\frac{1}{|\mathcal{S}_{0}|} \sum_{v_j \in \mathcal{S}_{0}} |\delta^{(0)}_{j,i}| + \frac{1}{|\mathcal{S}_{1}|} \sum_{v_j \in \mathcal{S}_{1}}  |\delta^{(1)}_{j,i}|)
\end{split}
\end{equation}
Therefore, Equation~\ref{eq:defs2} is bounded by
\begin{equation}
    \begin{split}
        &\left|\frac{1}{|\mathcal{S}_{0}|} \sum_{v_j \in \mathcal{S}_{0}} \sigma(z_{j,i})-\frac{1}{|\mathcal{S}_{1}|} \sum_{v_j \in \mathcal{S}_{1}} \sigma(z_{j,i})\right|\\
        \leq &|\operatorname{\sigma}(\bar{z}_i^{(0)}) - \operatorname{\sigma}(\bar{z}_i^{(1)})| + \\
        &\qquad L(\left|\frac{1}{|\mathcal{S}_{0}|} \sum_{v_j \in \mathcal{S}_{0}} |\delta^{(0)}_{j,i}|\right| +\left|\frac{1}{|\mathcal{S}_{1}|} \sum_{v_j \in \mathcal{S}_{1}}  |\delta^{(1)}_{j,i}|\right|) \\
        \leq &|\operatorname{\sigma}(\bar{z}_i^{(0)}) - \operatorname{\sigma}(\bar{z}_i^{(1)})|+2L\Delta_{z,i}\\
        \leq & L |\bar{z}_i^{(0)} -\bar{z}_i^{(1)}| + 2L\Delta_{z,i}.
    \end{split}
\end{equation}

Up to this point, the proof is exactly the same as Lemma~\ref{lemma:sigmoid} for each dimension $i\in[F_\ell]$. Taking the L2 norm of vectors, we have $\norm{\bm{\delta}_h^\ell}_2 \le L \norm{\bar{\bm{z}}^{\ell,(0)} -\bar{\bm{z}}^{\ell,(1)}}_2+2L\norm{\Delta_z^\ell}_2$.

\subsection{Proof of Error Terms}\label{app:error}

\new{
Let $\mat{V}, \mat{\Lambda}$ be the quantized eigenvectors and eigenvalues provided by the prover. The zkSNARK ensures that the following two equations hold: 

\begin{align}
    \mat{A} &= \mat{V}\cdot\mat{\Lambda}\cdot \mat{V}^T+\mat{E}\label{eq:a2}\\
    \mat{V}\cdot\mat{V}^T &= \mat{I}+\mat{E}'\label{eq:v2}
\end{align}

Let $\mat{Q}$ be the matrix where the columns are the accurate eigenvectors of $\mat{A}$ and $\mat{\Lambda}'$ be the diagonal matrix defined by the accurate eigenvalues. As $\mat{A} = \mat{W}\cdot \mat{W}^T$ is a symmetric positive semi-definite matrix, its eigenvalues and eigenvectors satisfy $\mat{A} = \mat{Q}\mat{\Lambda}'\mat{Q}^T$ and $\mat{Q}\cdot\mat{Q}^T = \mat{I}$. Hence \eqref{eq:a2} can be re-written as
\begin{align}
\mat{A}=\mat{Q}\cdot\mat{\Lambda}'\cdot\mat{Q}^T = \mat{V}\cdot\mat{\Lambda}\cdot \mat{V}^T+\mat{E}\label{eq:a3}
    \end{align}

As columns of $\mat{Q}$ form an orthonormal basis, $\mat{V}$ can be written as $\mat{V} = \mat{Q} \mat{C}$, where $\mat{C}=\mat{Q}^T\cdot \mat{V}$. Substituting into \eqref{eq:v2}, we have  
\begin{align}\label{eq:error}
    \mat{Q}\cdot\mat{C}\cdot\mat{C}^T\cdot\mat{Q}^T =& \mat{I}+\mat{E}'\nonumber\\
    \Rightarrow \mat{C}\cdot\mat{C}^T = &\mat{I}+\mat{Q}^T\mat{E}'\mat{Q}
\end{align}

Based on \eqref{eq:a3}, we have
\begin{align*}  
    \mat{\Lambda}' = &\mat{Q}^T\mat{V}\cdot\mat{\Lambda}\cdot \mat{V}^T\mat{Q}+\mat{Q}^T\cdot\mat{E}\cdot\mat{Q}\\
     = &\mat{C}\cdot\mat{\Lambda}\cdot \mat{C}^T+\mat{Q}^T\cdot\mat{E}\cdot\mat{Q}
\end{align*}
Rearranging the terms leads to 
\begin{align*}
\mat{Q}^T\cdot\mat{E}\cdot\mat{Q} = \mat{\Lambda}'- \mat{C}\cdot\mat{\Lambda}\cdot \mat{C}^T
\end{align*}
Taking Frobenious norms on both sides leads to
\begin{align*}
 \|\mat{E} \|^2_F = & \| \mat{\Lambda}'-\mat{C}\mat{\Lambda}\mat{C}^T\|_F\\
 = &\sum_{i=1}^N (\lambda_i'-\lambda_i\mat{c}_i^\top \mat{c}_i)^2+\sum_{i=1}^N \sum_{j\neq i} (\mat{c}_i^\top \lambda_i \mat{c}_j^\top)^2
\end{align*}

Hence, the following equality holds:
\begin{align}\label{eq:lambda2}
   &\sum_{i=1}^N (\lambda_i'-\lambda_i\mat{c}_i^\top \mat{c}_i)^2 \le \|\mat{E} \|^2_F \nonumber\\
   \Rightarrow& |\lambda_i'-\lambda_i\mat{c}_i^\top \mat{c}_i)| \le \|\mat{E} \|_F \quad \forall i\nonumber\\
   \Rightarrow &\lambda_i\mat{c}_i^\top \mat{c}_i-\|\mat{E}\|_F \le \lambda_i'\le \lambda_i\mat{c}_i^\top \mat{c}_i+\|\mat{E} \|_F \quad \forall i
\end{align}

According to \eqref{eq:error}, $\mat{c}_i^\top \mat{c}_i$ is the $i$-th value of $C\cdot C^T$ on the diagonal, and is thus equal to $1+e_{ii}''$, where $e_{ii}''$ is the $i$-th element on the diagonal of $\mat{Q}^T\mat{E}'\mat{Q}$. By the Cauchy-Schwarz inequality, $|e_{ii}''|\le \norm{\bm{q}_i}_2^2 \norm{\mat{E}'}_2 =  \norm{\mat{E}'}_2$. Combining with \eqref{eq:lambda2} we have
\begin{align*}
     \lambda_i(1-\norm{\mat{E}'}_2)-\|\mat{E}\|_F \le \lambda_i'\le \lambda_i(1+\norm{\mat{E}'}_2) +\|\mat{E}\|_F\quad \forall i
\end{align*}

Hence if both $\norm{\mat{E}'}_2$ and $\|\mat{E}\|_F$ are small, $\lambda_i$ is close to $\lambda_i'$, the real eigenvalue. 
}

\section{zkSNARK for Computing Aggregated Statistics}\label{app:delta}

\begin{figure}[h!]
   \centering
    \begin{minipage}{\linewidth}
   
   \begin{algorithm}[H]

\raggedright
	\textbf{Public input:} commitment $\ms{com}_{\mat{D}}$ of a dataset $\mat{D}\in\F^{N\times F}$ \\
        \textbf{Witness:} the dataset $\mat{D}$\\
	\textbf{Output:} $\bm{\delta}_x,\Delta_x$ 
	\begin{algorithmic}[1]		
        \State $\ms{Quantize\_Validate}(\mat{D})$ 
        \State $|\S_0|=0$
        \For{$j=0,1,\ldots, N-1$}
                \State $|\S_0| = (1-s_j) + |\S_0|$
        \EndFor
        \State $|\S_1|= N-|\S_0|$
        \For{$i=0,1,\ldots, F-1$}
            \State $\bar{x}^{(0)}_i=\bar{x}^{(1)}_i=0$
            \For{$j=0,1,\ldots, N-1$}
                \State $\bar{x}^{(0)}_i = x_{j,i}\cdot (1-s_j) + \bar{x}^{(0)}_i$
                \State $\bar{x}^{(1)}_i = x_{j,i}\cdot s_j + \bar{x}^{(1)}_i$
            \EndFor
                \State $\bar{x}^{(0)}_i = \frac{1}{|\S_0|}\bar{x}^{(0)}_i$, $\bar{x}^{(1)}_i = \frac{1}{|\S_1|}\bar{x}^{(1)}_i$
        \EndFor

        \For{$i=0,1,\ldots, F-1$}
            \For{$j=0,1,\ldots, N-1$}
                \State $x_{j,i}=x_{j,i}-\bar{x}_i^{(0)}(1-s_j)-\bar{x}_i^{(1)}s_j$
                \State $|x_{j,i}| = \ms{Absolute}(x_{j,i})$
            \EndFor
            \State $\Delta_{x,i} = \max(|\bm{x}_i|)$. 
        \EndFor

        \State Output $\bm{\delta}_x = (\bar{x}^{(0)}_i-\bar{x}^{(1)}_i)_{i\in[F]}, \Delta_x = (\Delta_{x,i})_{i\in[F]}$. 
	\end{algorithmic}
	\end{algorithm}
 \end{minipage}

    \caption{zkSNARK for computing $\bm{\delta}_x,\Delta_x$ of a dataset. }
    \label{fig:delta}
\end{figure}

\clearpage

\section{Meta-Review}

The following meta-review was prepared by the program committee for the 2025
IEEE Symposium on Security and Privacy (S\&P) as part of the review process as
detailed in the call for papers.

\subsection{Summary}
The paper proposes FairZK, a system for proving the fairness of ML models using zero-knowledge proofs (ZKPs). The system adopts a novel approach based on new fairness bounds that depend on model parameters and aggregate statistics of the dataset used. This results in much better performance and increased scalability compared to prior solutions that rely on ZKPs for inference to establish the model's fairness.

\subsection{Scientific Contributions}
\begin{itemize}
\item Creates a New Tool to Enable Future Science
\item Provides a Valuable Step Forward in an Established Field
\end{itemize}

\subsection{Reasons for Acceptance}
\begin{enumerate}
\item The proposed system adopts a novel approach for proving ML model fairness and significantly outperforms prior work.
\item The inferred bounds may be of independent interest and may inspire future research for tighter results.
\end{enumerate}

\subsection{Noteworthy Concerns} 
While the proposed scheme achieves better bounds (in absolute values) than prior works for provable fairness, it is not clear exactly what these values mean in terms of how ``fair'' the result is. This is a broader issue which is not specific to the proposed scheme, but the current evaluation does not show how sensitive these achieved bounds are with respect to different datasets and/or mode parameters.

\end{document}